\documentclass[11pt, onecolumn, journal,compsoc]{IEEEtran}

\usepackage{amsmath, amssymb,amsthm}
\usepackage{stmaryrd}
\usepackage{epsfig,calc, graphicx}
\usepackage{wasysym,pifont}
\usepackage{hyperref}
\usepackage[usenames]{color}

\title{Convergence radius and sample complexity of ITKM algorithms for dictionary learning}

\author{Karin Schnass
\thanks{Karin Schnass is with the Department of Mathematics, University of Innsbruck, Technikerstra\ss e 13, 6020 Innsbruck, Austria,
karin.schnass@uibk.ac.at.}
}

\newcommand\ip[2]{\langle #1, #2\rangle}

\newcommand\sparsity{S}
\newcommand\ddim{d}

\newcommand\eps{\varepsilon}

\newcommand\dico{\Phi}
\newcommand\dicoset{\mathcal{D}}
\newcommand\atom{\phi}

\newcommand\pdico{\Psi}
\newcommand\ppdico{\bar\Psi}
\newcommand\patom{\psi}
\newcommand\patomn{\bar\psi}

\newcommand\noise{r}

\newcommand\nsigma{\rho}

\newcommand\deltaz{{\delta_0}}

\newcommand\Sset{\mathbb{S}} 
 
\newcommand\factor{{\kappa}}
\newcommand\epstarget{{\tilde\eps}}

\newcommand\diag{\operatorname{diag}}
\newcommand\signop{\operatorname{sign}}
\newcommand\tr{\operatorname{tr}}

\newcommand{\R}{{\mathbb{R}}}

\newcommand{\E}{{\mathbb{E}}}
\newcommand{\I}{{\mathbb{I}}}
\renewcommand{\P}{{\mathbb{P}}}

\theoremstyle{plain}
\newtheorem{theorem}{Theorem}[section]
\newtheorem*{otheorem}{Theorem - O}
\newtheorem{corollary}[theorem]{Corollary}
\newtheorem*{ocorollary}{Corollary - O}

\newtheorem{lemma}[theorem]{Lemma}
\newtheorem{sublemma}[theorem]{Sublemma}

\newtheorem{definition}{Definition}[section]

\newtheorem{algorithm}[theorem]{Algorithm}

\theoremstyle{remark}
\newtheorem{example}{Example}[section]
\newtheorem{remark}[example]{Remark}

\begin{document}

\maketitle

\begin{abstract}In this work we show that iterative thresholding and K means (ITKM) algorithms can recover a generating dictionary with K atoms from noisy $S$ sparse signals up to an error $\epstarget$ as long as the initialisation is within a convergence radius, that is up to a $\log K$ factor inversely proportional to the dynamic range of the signals, and the sample size is proportional to $K \log K \epstarget^{-2}$. The results are valid for arbitrary target errors if the sparsity level is of the order of the square root of the signal dimension $d$ and for target errors down to $K^{-\ell}$ if $S$ scales as $S \leq d/(\ell \log K)$. 
\end{abstract}


\section{Introduction}\label{sec:intro}
The goal of dictionary learning is to find a dictionary that will sparsely represent a class of signals.
That is given a set of $N$ training signals $y_n\in \R^d$, which are stored as columns in a matrix $Y=(y_1 ,\ldots ,y_N)$, one wants to find a collection of $K$ normalised vectors $\atom_k\in \R^d$, called atoms, which are stored as columns in the dictionary matrix $\dico=(\atom_1, \ldots, \atom_K) \in \R^{d\times K}$, and coefficients $x_n$, which are stored as columns in the coefficient matrix $X=(x_1, \ldots ,x_N)$ such that
\begin{align}\label{abstractdl}
 \quad Y=\dico X \quad \mbox{and}\quad X \mbox{ sparse}.
\end{align}
Research into dictionary learning comes in two flavours corresponding to the two origins of the problem, the slightly older one in the independent component analysis (ICA) and blind source separation (BSS) community, where dictionary learning is also known as sparse component analysis, and the slightly younger one in the signal processing community, where it is also known as sparse coding.
The main motivation for dictionary learning in the ICA/BSS community comes from the assumption that the signals of interest are generated as sparse mixtures - random sparse mixing coefficients $X_0$ - of several sources or independent components - the dictionary $\dico_0$ - which can be used to describe or explain a (natural) phenomenon, \cite{olsfield96, lese00,krra00, kreutz03}. For instance in the 1996 paper by Olshausen and Field, \cite{olsfield96}, which is widely regarded as the mother contribution to dictionary learning, the dictionary is learned on patches of natural images, and the resulting atoms bear a striking similarity to simple cell receptive fields in the visual cortex. A natural question in this context is, when the generating dictionary $\dico_0$ can be identified from $Y$, that is, the sources from the mixtures. Therefore the first theoretical insights into dictionary learning came from this community, \cite{gethci05}. Also the first dictionary recovery algorithms with global success guarantees, which are based on finding overlapping clusters in a graph derived from the signal correlation matrix $Y^\star Y$, take the ICA/BSS point of view, \cite{argemo13, aganne13}.\\
The main motivation for dictionary learning in the signal processing community is that sparse signals are immensely practical, as they can be easily stored, denoised, or reconstructed from incomplete information, \cite{doelte06,masazi08,mabapo12}. Thus the interest is less in the dictionary itself but in the fact that it will provide sparse representations $X$. Following the rule 'the sparser - the better' the obvious next step is to look for the dictionary that provides the sparsest representations. So given a budget of $K$ atoms and $S$ non-zero coefficients per signal, one way to concretise the abstract formulation of the dictionary learning problem in \eqref{abstractdl} is to formulate it as optimisation problem, such as
\begin{align}
(P_{2,S})\qquad \min \|Y- \dico X\|_F \quad \mbox{s.t.} \quad \|x_n\|_0\leq S \quad \mbox{and} \quad \dico\in \dicoset,
\end{align}
where $ \|\cdot\|_0$ counts the nonzero elements of a vector or matrix and $\dicoset$ is defined as $\dicoset=\{\dico=(\atom_1, \dots ,\atom_K) : \|\atom_k\|_2 =1\}$.  While $(P_{2,S})$ is for instance the starting point for the MOD or K-SVD algorithms, \cite{enaahu99, ahelbr06}, other definitions of \emph{optimally} sparse lead to other optimisation problems and algorithms, \cite{zipe01,pl07, yablda09, mabaposa10, sken10, rubrel10}. The main challenge of optimisation programmes for dictionary learning is finding the global optimum, which is hard because the constraint manifold $\dicoset$ is not convex and the objective function is invariant under sign changes and permutations of the dictionary atoms with corresponding sign changes and permutations of the coefficient rows. In other words for every local optimum there are $2^K K!-1$ equivalent local optima.\\
So while in the signal processing setting there is a priori no concept of a generating dictionary, it is often used as auxiliary assumption to get theoretical insights into the optimisation problem. Indeed without the assumption that the signals are sparse in some dictionary the optimisation formulation makes little or no sense. For instance if the signals are uniformly distributed on the sphere in $\R^d$, in asymptotics $(P_{2,S})$ becomes a covering problem and the set of optima is invariant under orthonormal transforms. \\
Based on a generating model on the other hand it is possible to gain several theoretical insights. For instance, how many training signals are necessary such that the sparse representation properties of a dictionary on the training samples (e.g. the optimiser) will extrapolate to the whole class, \cite{mapo10, vamabr11, megr12, grjebaklse13}. What are the properties of a generating dictionary and the maximum sparsity level of the coefficients and signal noise such that this dictionary is a local optimiser or near a local optimiser given enough training signals, \cite{grsc10, gewawrXX, sc14, sc14b, bagrje14}.\\
An open problem for overcomplete dictionaries with some first results for bases, \cite{spwawr12, suquwr15}, is whether there are any spurious optimisers which are not equivalent to the generating dictionary, or if any starting point of a descent algorithm will lead to a global optimum? A related question (in case there are spurious optima) is, if the generating dictionary is the global optimiser? If yes, it would justify using one of the graph clustering algorithms for recovering the optimum, \cite{argemo13, aganne13, arbhgema14, bakest14}. This is important since all dictionary learning algorithms with global success guarantees are computationally very costly, while optimisation approaches are locally very efficient and robust to noise. Knowledge of the convergence properties of a descent algorithm, such as convergence radius (basin of attraction), rate or limiting precision based on the number of training signals, therefore helps to decide when it should take over from a global algorithm for fast local refinement, \cite{aganjaneta13}. \\
In this paper we will investigate the convergence properties of two iterative thresholding and K-means algorithms. The first algorithm ITKsM, which uses signed signal means, originates from the response maximisation principle introduced in \cite{sc14b}. There it is shown that a generating $\mu$-coherent dictionary constitutes a local maximum of the response principle as long as the sparsity level of the signals scales as $S=O(\mu^{-1})$. It further contains the first results showing that the maximiser remains close to the generator for sparsity levels up to $S=O(\mu^{-2}/\log K)$. For a target recovery error $\tilde \eps$ the sample complexity $N$ is shown to scale as $N=O(SK^3 \tilde \eps^{-2})$ and the basin of attraction is conjectured to be of size $O(1/\sqrt{S})$. \\
Here we will not only improve on the conjecture by showing that in its online version the algorithm has a convergence radius of size $O(1/\sqrt{\log K})$ but also show that for the algorithm rather than the principle the sample complexity reduces to $N=O(K\log K \tilde \eps^{-2}\log(\eps^{-1}))$ (omitting $\log \log$ factors). Again recovery to arbitrary precision holds for sparsity levels $S=O(\mu^{-1})$ and stable recovery up to an error $K^{-\ell}$ for sparsity levels $S=O(\mu^{-2}/(\ell\log K))$. We also show that the computational complexity assuming an initialisation within the convergence radius scales as $O(\log(\epstarget^{-1})dKN)$ or omitting log factors $O(dK^2\epstarget^{-2})$.\\
Motivated by the desire to reduce the sample complexity for the case of exactly sparse, noiseless signals, we then introduce a second iterative thresholding and K-means algorithms ITKrM, which uses residual instead of signal means. 
It has roughly the same properties as ITKsM apart from the convergence radius which reduces to $O(1/\sqrt{S})$ and the computational complexity, which scales as $O(dN(K + S^2))$ and thus can go up to $O(d^2NK)$ for $S=O(d)$. However, if $S=O(\mu^{-1})$ and the signals follow an exactly sparse, noiseless model, we can show that the sample complexity reduces to $N=O (K \epstarget^{-1}\log(\eps^{-1}))$ (omitting $\log \log$ factors). Our results are in the same spirit as the results for the alternating minimisation algorithm in \cite{aganjaneta13} but have the advantage that they are valid for more general coefficient distributions and a lower level of sparsity (S larger) resp. higher level of coherence, that the convergence radius is larger and that the algorithms exhibit a lower computational complexity. They are also close to some very recent results about several alternating minimisation algorithms, which are like the ITKMs based on thresholding, \cite{argemamo15}. Compared to our results they are essentially the same in terms of convergence radius and sample complexity but are only valid for sparsity levels $S =O(\mu^{-1})$ and up to a limiting precision (even in the exact sparse noiseless case). More interestingly \cite{argemamo15} contains a strategy for finding initialisations within a radius $O(1/\log K)$ to the generating dictionary, which is proven to succeed for sparsity levels $S=O(\mu^{-1})$. With slight modifications and using Tropp's results on average isometry constants, \cite{tr08}, this initialisation strategy could probably be proven to work also for sparsity levels up to $S=O(\mu^{-2}/(\ell\log K))$. However, its computational complexity seems to explode as $S$ grows.\\
The rest of the paper is organised as follows. After summarising notation and conventions in the following section, in Section \ref{sec:itksm} we re-introduce the ITKsM algorithm, discuss our sparse signal model and analyse the convergence properties of ITKsM. Based on the shortcomings of ITKsM we motivate the ITKrM algorithm in Section \ref{sec:itkrm}, and again analyse its convergence properties. In Section \ref{sec:sim} we provide numerical simulations indicating that the convergence radius of both ITKM algorithms is generically much larger and that sometimes ITKrM even converges globally from random initialisations. Finally in Section \ref{sec:discussion} we compare our results to existing work and point out future directions of research.
%

\section{Notations and Conventions}\label{sec:notations}
Before we join the melee, we collect some definitions and lose a few words on notations; usually subscripted letters will denote vectors with the exception of $\eps, \alpha, \omega$, where they are numbers, eg. $x_n \in \R^K$ vs. $\eps_k \in \R$, however, it should always be clear from the context what we are dealing with. \\
For a matrix $M$, we denote its (conjugate) transpose by $M^\star$ and its Moore-Penrose pseudo inverse by $M^\dagger$. We denote its operator norm by $\|M\|_{2,2}=\max_{\|x\|_2=1}\|Mx\|_2$ and its Frobenius norm by $\|M\|_F= \tr(M^\star M)^{1/2}$, remember that we have $\|M\|_{2,2}\leq \|M\|_F$.\\
We consider a {\bf dictionary} $\dico$ a collection of $K$ unit norm vectors $\atom_k\in \R^d$, $\|\atom_k\|_2=1$. By abuse of notation we will also refer to the $d \times K$ matrix collecting the atoms as its columns as the dictionary, i.e. $\dico=(\atom_1, \ldots \atom_K)$. The maximal absolute inner product between two different atoms is called the {\bf coherence} $\mu$ of a dictionary, $\mu=\max_{k \neq j}|\ip{\atom_k}{\atom_j}|$.\\
By $\dico_I$ we denote the restriction of the dictionary to the atoms indexed by $I$, i.e. $\dico_I=(\atom_{i_1},\ldots, \atom_{i_\sparsity} )$, $i_j\in I$, and by $P(\dico_I)$ the orthogonal projection onto the span of the atoms indexed by $I$, i.e. $P(\dico_I)=\dico_I \dico_I^\dagger$. Note that in case the atoms indexed by $I$ are linearly independent we have $\dico_I^\dagger = (\dico_I^\star \dico_I)^{-1} \dico_I^\star$. We also define $Q(\dico_I)$ the orthogonal projection onto the orthogonal complement of the span on $\dico_I$, that is $Q(\dico_I) = \I_d - P(\dico_I)$, where $\I_d$ is the identity operator (matrix) in $\R^d$.\\
(Ab)using the language of compressed sensing we define $\delta_I(\dico)$ as the smallest number such that all eigenvalues of $\dico^\star_I\dico_I$ are included in $[1-\delta_I(\dico), 1+\delta_I(\dico)]$ and the {\bf isometry constant} $\delta_S(\dico)$ of the dictionary as $\delta_S(\dico):=\max_{|I|\leq S} \delta_I(\dico)$. 
When clear from the context we will usually omit the reference to the dictionary. For more details on isometry constants, see for instance \cite{carota06}.\\
To keep the sub(sub)scripts under control we denote the {\bf indicator function of a set} $\mathcal V$ by $\chi(\mathcal V,\cdot)$, that is $\chi(\mathcal V, v)$ is one if $v \in \mathcal V$ and zero else. The set of the first $S$ integers we abbreviate by $\mathbb{S} = \{1,\ldots, S\}$.\\
We define the {\bf distance} of a dictionary $\pdico$ to a dictionary $\dico$ as
\begin{align}
d(\dico,\pdico):=\max_k \min_\ell \|\atom_k \pm \patom_\ell\|_2 = \max_k \min_\ell \sqrt{2-2|\ip{\atom_k}{\patom_\ell}|}.
\end{align}
Note that this distance is not a metric, since it is not symmetric. For example if $\dico$ is the canonical basis and $\pdico$ is defined by $\patom_i=\atom_i$ for $i\geq 3$, $\patom_1=(e_1 + e_2)/\sqrt{2}$, and $\patom_2=\sum_i \atom_1/\sqrt{d}$ then we have $d(\dico,\pdico)= 1/\sqrt{2}$ while $d(\pdico,\dico)=\sqrt{2-2/\sqrt{d}}$. 
A {\bf symmetric distance} between two dictionaries $\dico,\pdico$ could be defined as the maximal distance between two corresponding atoms, i.e.
\begin{align}
d_s(\dico,\pdico):=\min_{p \in \mathcal P} \max_k \|\atom_k\pm \patom_{p(k)}\|_2,
\end{align}
where $\mathcal P$ is the set of permutations of $\{1,\ldots, S\}$. Since locally the distances are equivalent we will state our results in terms of the easier to calculate asymmetric distance and assume that $\pdico$ is already signed and rearranged in a way that $d(\dico,\pdico)=\max_k \|\atom_k-\patom_k\|_2$. \\
We will make heavy use of the following decomposition of a dictionary $\pdico$ into a given dictionary $\dico$ and a perturbation dictionary $Z$. If $d(\pdico,\dico)=\eps$ we set $\|\patom_k - \atom_k\|_2=\eps_k$, where by definition $\max_k \eps_k = \eps$. We can then find unit vectors $z_k$ with $\langle \atom_k,z_k\rangle = 0$ such that 
\begin{align}\label{atomdecomp}
\patom_k = \alpha_k \atom_k + \omega_k z_k, \quad \mbox{for}, \quad\alpha_k:= 1-\eps^2_k/2  \quad\mbox {and} \quad\omega_k := (\eps_k^2 - \eps_k^4/4)^{\frac{1}{2}}.
\end{align}
The dictionary $Z$ collects the perturbation vectors on its columns, that is $Z=(z_1, \ldots z_K)$ and we define the diagonal matrices $A_I, W_I$ implicitly via 
\begin{align}\label{dicodecomp}
\pdico_I = \dico_I A_I + Z_I W_I,
\end{align} or in MATLAB notation $A_I= \diag (\alpha_I)$ with $\alpha_I=(\alpha_k)_{k\in I}$ and analogue for $W_I$.
Based on this decomposition we further introduce the short hand $b_k= \frac{\omega_k}{\alpha_k} z_k$ and $B_I= Z_I W_I A_I^{-1}$.\\
We consider a {\bf frame} $F$ a collection of $K\geq d$ vectors $f_k\in\R^\ddim$ for which there exist two positive constants $A,B$ such that for all $v \in \R^\ddim$ we have
\begin{align}\label{framebound}
A \|v\|^2_2 \leq \sum_{k=1}^K |\ip{f_k}{v}|^2 \leq B \|v\|^2_2.
\end{align}
If $B$ can be chosen equal to $A$, i.e. $B=A$, the frame is called tight and if all elements of a tight frame have unit norm we have $B=A=K/\ddim$. The operator $FF^\star$ is called frame operator and by \eqref{framebound} its spectrum is bounded by $A, B$. For more details on frames, see e.g. \cite{ch03}.\\
Finally we introduce the Landau symbols $O,o$ to characterise the growth of a function. We write 
\begin{align}
 f(t)=O(g(t)) &\quad\mbox{ if }\quad \lim_{t \rightarrow 0/\infty} f(t)/g(t)= C<\infty \notag \\
\mbox{and}\quad f(t)=o(g(t))&\quad \mbox{ if }\quad \lim_{t \rightarrow 0/\infty} f(\eps)/g(\eps)=0.\notag
\end{align}

\section{Dictionary Learning via ITKsM}\label{sec:itksm}
Iterative thresholding and K signal means (ITKsM) for dictionary learning was introduced 
as algorithm to maximise the $S$-response criterion
\begin{align}\label{thecriterion}
(P_{R1}) \hspace{2cm} \max_{\pdico \in \dicoset} \sum_n \max_{|I| = S} \| \pdico_I^\star y_n\|_1,
\end{align}
which for $S=1$ reduces to the K-means criterion, \cite{sc14b}. It belongs to the class of alternating optimisation algorithms for dictionary learning, which alternate between updating the sparse coefficients based on the current version of the dictionary and updating the dictionary based on the current version of the coefficients, \cite{enaahu99, ahelbr06, aganjaneta13}. As its name suggests, the update of the sparse coefficients is based on thresholding while the update of the dictionary is based on K signal means.
\begin{algorithm}[ITKsM one iteration] 
Given an input dictionary $\pdico$ and $N$ training signals $y_n$ do:
\begin{itemize}
\item For all $n$ find $I_{\pdico,n}^t= \arg\max_{I: | I |=S} \| \pdico_I^\star y_n\|_1$.
\item For all $k$ calculate
\begin{align}\label{itksm_batchsum}
\patomn_k=\frac{1}{N} \sum_n y_n \cdot \signop(\ip{\patom_k}{y_n}) \cdot  \chi(I_{\pdico,n}^t, k).
\end{align}
\item Output $\ppdico=(\patomn_1/\|\patomn_1\|_2, \ldots, \patomn_K/\|\patomn_K\|_2)$.
\end{itemize}
\end{algorithm}
The algorithm can be stopped after a fixed number of iterations or once a stopping criterion, such as improvement $d(\ppdico,\pdico)\leq \theta$ for some threshold $\theta$, is reached. Its advantages over most other dictionary learning algorithms are threefold. First it has very low computational complexity. In each step the most costly operation is the calculation of the $N$ matrix vector products $\pdico^\star y_n$, that is the matrix product $\pdico^\star Y$, of order $O(dKN)$. In comparison the globally successful graph clustering algorithms need to calculate the signal correlation matrix $Y^\star Y$, cost $O(dN^2)$.\\ 
Second due to its structure only one signal has to be processed at a time. Instead of calculating $I_{n}^t$ for all $n$ and calculating the sum, one simply calculates $I_{\pdico,n}^t$ for the signal at hand, updates all atoms $\patomn_k$ for which $k\in I_{\pdico,n}^t$ as $\patomn_k \rightarrow \patomn_k + y_n \cdot \signop(\ip{\patom_k}{y_n})$ and turns to the next signal. Once $N$ signals have been processed one does the normalisation step and outputs $\ppdico$. Further in this online version only $(2K+1)d$ values corresponding to the input dictionary, the current version of the updated dictionary and the signal at hand, need to be stored rather than the $N\times d$ signal matrix. Parallelisation can be achieved in a similar way. Again for comparison, the graph clustering algorithms, K-SVD, \cite{ahelbr06}, and the alternating minimisation algorithm in \cite{aganjaneta13} need to store the whole signal resp. residual matrix as well as the dictionary.\\
The third advantage is that with high probability the algorithm converges locally to a generating dictionary $\dico$ assuming that we have enough training signals and that these follow a sparse random model in $\dico$. In order to prove the corresponding result we next introduce our sparse signal model. 

\subsection{Signal Model}\label{sec:signalmodel}
We employ the same signal model, which has already been used for the analyses of the S-response and K-SVD principles, \cite{sc14, sc14b}. 
Given a $d\times K$ dictionary $\dico$, we assume that the signals are generated as, 
\begin{align}\label{noisymodel1}
y=\frac{ \dico x +\noise}{\sqrt{1+\|\noise \|_2^2}},
\end{align}
where $x$ is drawn from a sign and permutation invariant probability distribution $\nu$ on the unit sphere $S^{K-1} \subset \R^K$ and $\noise=(\noise(1) \ldots \noise(d))$ is a centred random subgaussian vector with parameter $\nsigma$, that is $\E(\noise) = 0$ and for all vectors $v$ the marginals $\ip{v}{\noise}$ are subgaussian with parameter $\nsigma$, meaning they satisfy $\E (e^{t \ip{v}{\noise}}) \leq e^{t^2 \nsigma^2/2}$ for all $t>0$.
We recall that a probability measure $\nu$ on the unit sphere is sign and permutation invariant, if for all measurable sets $\mathcal{X}\subseteq S^{K-1}$, for all sign sequences $\sigma \in \{-1,1\}^d$ and all permutations $p$ we have
\begin{align}
\nu( \sigma \mathcal X)=\nu(\mathcal X), \quad &\mbox{where} \quad \sigma \mathcal X := \{ (\sigma(1) x(1), \ldots, \sigma(K) x(d) ) : x \in \mathcal{X} \}\\
\nu( p( \mathcal X))=\nu(\mathcal X), \quad &\mbox{where} \quad p(\mathcal X ) := \{ ( x(p(1)), \ldots, x(p(K)) ) : x \in \mathcal{X} \}.
\end{align}
We can get a simple example of such a measure by taking a positive, non increasing sequence $c$, that is $c(1) \geq c(2) \geq \ldots \geq c(K) \geq 0$, choosing a sign sequence $\sigma$ and a permutation $p$ uniformly at random and setting $x=x_{p,\sigma}$ with $x_{p, \sigma}(k)=\sigma(k) c(p(k))$. Conversely we can factorise any sign and permutation invariant measure into a random draw of signs and permutations and a measure on the space of non-increasing sequences. \\
By abuse of notation let $c$ now denote the mapping that assigns to each $x \in S^{K-1}$ the non increasing rearrangement of the absolute values of its components, i.e. $c: x \rightarrow c_x$ with $c_x(k): = |x({p(k)})| $ for a permutation $p$ such that $|x({p(1)})| \geq |x({p(2)})| \geq \ldots \geq |x({p(K)})|  \geq 0$. Then the mapping $c$ together with the probability measure $\nu$ on $x \in S^{K-1}$ induces a probability measure $\nu_c$ on $c(S^{K-1})=S^{K-1}\cap [0,1]^K$ via the preimage $c^{-1}$, that is $\nu_c(\Omega) := \nu(c^{-1}(\Omega))$ for any measurable set $\Omega \subseteq c(S^{K-1})$. \\
Using this new measure we can rewrite our signal model as
\begin{align}\label{noisymodel2}
y=\frac{ \dico x_{c, p,\sigma} +\noise}{\sqrt{1+\|\noise \|_2^2}},
\end{align}
where we define $x_{c, p, \sigma}(k)= \sigma(k) c(p(k))$ for a positive, non-increasing sequence $c$ distributed according to $\nu_c$, a sign sequence $\sigma$ and a permutation $p$ distributed uniformly at random and $\noise$ again a centred random subgaussian vector with parameter $\nsigma$. Note that we have $\E(\|r\|_2^2)\leq d\nsigma^2$, with equality for instance in the case of Gaussian noise.
To incorporate sparsity into our signal model we make the following definitions.
\begin{definition}
A sign and permutation invariant coefficient distribution $\nu$ on the unit sphere $S^{K-1} \subset \R^K$ is called $S$-sparse with absolute gap $\beta_S> 0$ and relative gap $\Delta_S>\beta_S$, if 
\begin{align}
\nu\left( c_x(S)-c_x(S+1) < \beta_S\right)=0 \qquad \mbox{and} \qquad \nu\left( \frac{c_x(S)-c_x(S+1)}{c_x(1)} < \Delta_S \right)=0,
\end{align}
or equivalently
\begin{align}
\nu_c\left( c(S)-c(S+1) < \beta_S\right)=0 \qquad \mbox{and} \qquad \nu_c\left( \frac{c(S)-c(S+1)}{c(1)} < \Delta_S \right)=0.
\end{align}
The coefficient distribution is called strongly $S$-sparse if $\Delta_S \geq 2\mu S$.
\end{definition}
For exactly sparse signals $\beta_S$ is simply the smallest non-zero coefficient and $\Delta_S$ is the inverse dynamic range of the non-zero coefficients. We have the bounds $\beta_S\leq \frac{1}{\sqrt{S}}$
and $\Delta_S \leq 1$. Since equality holds for the 'flat' distribution generated from $c(k) =\frac{1}{\sqrt{S}}$ for $k\leq S$ and zero else, we will usually think of $\beta_S$ being of the order $O(\frac{1}{\sqrt{S}})$ and $\Delta_S$ being of the order $O(1)$.
We can also see that coefficient distributions can only be strongly $S$-sparse as long as $S$ is smaller than $\frac{\Delta_S}{2\mu}$, that is $S = O(\mu^{-1}) = O(\sqrt{d})$.\\ 
For the statement of our results we will use three other signal statistics, 
\begin{align}
\gamma_{1,S} :=  \E_c\left(c(1)+ \ldots + c(S)\right) \qquad \gamma_{2,S} :=  \E_c\left(c^2(1)+ \ldots + c^2(S)\right) \qquad C_r := \E_r\left( \frac{1}{\sqrt{1+\|r\|_2^2}}\right).
\end{align}
The constants $\gamma_{1,S}$ and $C^2_r$ will help characterise the expected size of $\patomn_k$. 
We have $S\beta_S \leq \gamma_{1,S} \leq \sqrt{S}$ and 
\begin{align}
C_r \geq \frac{1- e^{-d}}{\sqrt{1+5d\nsigma^2}},
\end{align}
compare \cite{sc14b}.
From the above inequality we can see that $C_r$ captures the expected signal to noise ratio, that is for large $\nsigma$ we have
\begin{align}
C^2_r \approx \frac{1}{d\nsigma^2} \approx \frac{\E(\| \dico x \|_2^2 ) }{\E(\|r\|_2^2) }.
\end{align}
Similarly the constant $\gamma_{2,S}$ can be interpreted as the expected energy of the signal approximation using the largest $S$ generating coefficients and the generating dictionary, or in other words $1-\gamma_{2,S}$ is a bound for the expected energy of the approximation error. \\
For noiseless signals generated from the flat distribution described above we have $\gamma_{1,S}=\sqrt{S}$, $C_r=1$ and $\gamma_{2,S}=1$, so we will usually think of these constants having the orders $\gamma_{1,S}=O(\sqrt{S})$, $C_r=O(1)$ and $\gamma_{2,S}=O(1)$.\\
From the discussion we see that, while being relatively simple, our signal model allows us to capture both approximation error and noise. Our results  have quite straightforward extensions to more complicated (realistic) signal models, which for instance include outliers (normalised but not sign or permutation invariant coefficients) or a small portion of coefficients without gap. With somewhat more effort it is also possible to relax the assumption of sign and permutation invariance in our coefficient model, potentially at the cost of decreasing the admissible sparsity level, the convergence radius and the recovery accuracy and increasing the sample complexity. Indeed we will see that the main reason for assuming sign invariance is to ensure that when thresholding with the generating dictionary always succeeds in recovering the generating support with a large margin and therefore also succeeds with a perturbed dictionary. To a lesser degree, especially in the case of ITKrM, the sign invariance also supports the permutation invariance in ensuring a richness of signals such that the averaging procedures contract towards the generating atoms. In particular the permutation invariance prevents the situation that two atoms are always used together and could therefore be replaced by two of their linear combinations.\\
However, we will sacrifice generality for comprehensibility and therefore just give pointers in the respective proofs.

\subsection{Convergence analysis of ITKsM}
We first look at the more general case of noisy, non exactly S-sparse signals and specialise to noiseless, strongly S-sparse signals later.
\begin{theorem} \label{th:itksm}
Let $\dico$ be a unit norm frame with frame constants $A\leq B$ and coherence $\mu$ and assume that the training signals $y_n$ are generated according to the signal model in \eqref{noisymodel2} with coefficients that are $S$-sparse with absolute gap $\beta_S$ and relative gap $\Delta_S$. \\
Fix a target error $\epstarget\geq 4 \eps_{\mu, \nsigma}$, where 
\begin{align}\label{epsmin}
\eps_{\mu, \nsigma} := \frac{8K^2\sqrt{B+1}}{C_r \gamma_{1,S}} \exp\left(\frac{-\beta_S^2}{98\max\{ \mu^2,\nsigma^2\}}\right).
\end{align}
Given an input dictionary $\pdico$ such that 
\begin{align}
d(\pdico, \dico)\leq \frac{\Delta_S}{\sqrt{98B}\left(\frac{1}{4} +\sqrt{\log\left(\frac{1060K^2(B+1)}{\Delta_S C_r \gamma_{1,S}}\right)}\right)},
\end{align}
then after $6 \lceil \log( \epstarget^{-1})\rceil$ iterations of ITKsM each on a fresh batch of $N$ training signals the output dictionary $\tilde \pdico$ satisfies 
\begin{align}
d( \tilde \pdico, \dico) &\leq \epstarget
\end{align}
except with probability 
\begin{align}
18 \lceil \log( \epstarget^{-1})\rceil K\exp\left(\frac{-C^2_r\gamma_{1,S}^2N\epstarget^2}{200 SK}\right).
\notag
\end{align} 
\end{theorem}

Before providing the proof let us discuss the result above. We first see that ITKsM will succeed if the input dictionary is within a radius $O(\Delta_S/\sqrt{\log K})$ to the generating dictionary $\dico$. In case of exactly sparse signals this means that the convergence radius is up to a log factor inversely proportional to the dynamic range of the coefficients.
This should not be come as a big surprise, considering that the average success of thresholding for sparse recovery with a ground truth dictionary depends on the dynamic range, \cite{scva07}. It also means that in the best case the convergence radius is actually of size $O(1/\sqrt{\log K})$, since for the flat distribution $\Delta_S=1$.\\
Next note that in the theorem above we have restricted the target error to be larger than $4\eps_{\mu, \nsigma}$. However at the cost of unattractively large constants in the probability bound, we can actually reach any target error larger than $\eps_{\mu, \nsigma}$. \\
To highlight the relation between the sparsity level and the minimally achievable error, we specialise the result to coefficients drawn from the flat distribution, meaning $\beta_S = 1/\sqrt{S}$. We further assume white Gaussian noise with variance $\nsigma^2 = 1/d$, corresponding to an expected signal to noise ratio of 1, and an incoherent dictionary with $\mu \leq 1/\sqrt{d}$. If $S\leq \frac{d}{98 \ell \log K}$ for some $\ell \geq 2$ then the minimally achievable error $\eps_{\mu, \nsigma}$ can be as small as $O(K^{2- \ell})$. \\
Last we want to get a feeling for the total number of training signals we need to have a good success probability. For exactly S-sparse signals with dynamic coefficient range 1 we have $\gamma_{1,S}=\sqrt{S}$. Omitting loglog factors each iteration is therefore likely to be successful when using a batch of $N=O(K\log K \epstarget^{-2})$ training signals, meaning that ITKsM is successful with high probability as soon as the total number of training signals used in the algorithms scales as $O(K\log K \epstarget^{-2} \log( \epstarget^{-1}))$. Note that in case of noise due to information theoretic arguments the factor $\epstarget^{-2}$ seems unavoidable, \cite{juelgo14}. \\
To summarise the discussion we provide an O-notation version of the theorem, which is less plug and play but free of messy constants and as such better suited to convey the quality of the result. Compare also Subsection~\ref{sec:signalmodel} for the O notation conventions.
\begin{otheorem}[\bf \ref{th:itksm}]Assume that in each iteration the number of training signals scales as $N=O(K\log K \epstarget^{-2})$. If $S\leq O(\frac{1}{\ell \mu^2 \log K})$ then with high probability for any starting dictionary $\pdico$ within distance $\eps\leq O(1/\sqrt{\log K})$ to the generating dictionary after $O( \log( \epstarget^{-1}))$ iterations of ITKsM, each on a fresh batch of training signals, the distance of the output dictionary $\tilde \pdico$ to the generating dictionary will be smaller than 
\begin{align}
\max \left\{\epstarget, O\left(K^{2-\ell}\right)\right\}.
\end{align}
\end{otheorem}

\begin{proof}
The proof consists of two steps. First we show that with high probability one iteration of ITKsM reduces the error by at least a factor $\factor<1$. Then we iteratively apply the results for one iteration.\\
\emph{ Step 1:}
For the first step we use the following ideas, compare also \cite{sc14b}: For most sign sequences $\sigma_n$ and therefore most signals \begin{align}
y_n=\frac{ \dico x_{c_n, p_n,\sigma_n} +\noise_n}{\sqrt{1+\|\noise_n \|_2^2}} \notag
\end{align} thresholding with a perturbation of the original dictionary will still recover the generating support $I_n := p_n^{-1}(\Sset)$, that is $I_{\pdico, n}^t=I_n$. Assuming that the generating support is recovered, for each $k$ the expected difference of the sum in~\eqref{itksm_batchsum} between using the original $\dico$ and the perturbation $\pdico$ is small, that is smaller than $d(\dico, \pdico)=\eps$, and due to concentration of measure also the difference on a finite number of samples will be small. Finally for each $k$ the sum in \eqref{itksm_batchsum} will again concentrate around its expectation, a scaled version of the atom $\atom_k$. \\
Formally we write,
\begin{align}
\bar \patom _k=\frac{1}{N} \sum_n &y_n \,\signop(\ip{\patom_k}{y_n}) \,  \chi(I_{\pdico,n}^t, k) - \frac{1}{N} \sum_n y_n \, \sigma_n(k) \, \chi(I_n,k)\\
&+ \frac{1}{N} \sum_n y_n \, \sigma_n(k) \, \chi(I_n,k) - \E\left( \frac{1}{N} \sum_n y_n \, \sigma_n(k) \,  \chi(I_n,k)\right) +  \E\left( \frac{1}{N} \sum_n y_n \, \sigma_n(k) \, \chi(I_n,k)\right).
\end{align}
Since $\E\left( \frac{1}{N} \sum_n y_n \, \sigma_n(k) \,  \chi(I_n,k)\right)= \frac{C_r\gamma_{1,S}}{K}\atom_k$, see the proof of Lemma~\ref{lemma2} in the appendix, using the triangle inequality and the bound $\|y_n\|_2\leq\sqrt{B+1}$ we get,
\begin{align}
\left\| \bar \patom_k - \frac{C_r\gamma_{1,S}}{K}\atom_k\right\|_2 &\leq  \left\| \frac{1}{N}\sum_n y_n \, [\signop(\ip{\patom_k}{y_n}) \, \chi(I_{\pdico,n}^t, k) -  \sigma_n(k) \, \chi(I_n,k)] \right\|_2\notag\\
&\hspace{5cm} + \left\|\frac{1}{N} \sum_n y_n \, \sigma_n(k) \, \chi(I_n,k) -  \frac{C_r\gamma_{1,S}}{K}\atom_k\right\|_2\notag\\
&\leq \frac{2 \sqrt{B+1}}{N}\, \sharp\{n: \signop(\ip{\patom_k}{y_n}) \, \chi(I_{\pdico,n}^t, k) \neq  \sigma_n(k) \, \chi(I_n,k)\}\notag\\
&\hspace{5cm} + \left\|\frac{1}{N} \sum_n y_n \, \sigma_n(k) \, \chi(I_n,k) -  \frac{C_r\gamma_{1,S}}{K}\atom_k\right\|_2. \label{itksmsplit}
\end{align}
Next note that for the draw of $y_n$ the event that for a given index $k$ the signal coefficient using thresholding with $\pdico$ is different from the oracle signal is contained in the event that thresholding does not recover the entire generating support $I_{\pdico,n}^t \neq I_n$ or that on the generating support the empirical sign pattern using $\pdico$ is different from the generating pattern, $\signop(\ip{\patom_k}{y_n} )\neq \sigma_n(k)$ for a $k\in I_n$,
\begin{align}
\{y_n :  \signop(\ip{\patom_k}{y_n}) \, \chi(I_{\pdico,n}^t, k) \neq  \sigma_n(k) \, \chi(I_n,k)\} \subseteq \{y_n : I_{\pdico,n}^t \neq I_n\} \cup \{y_n : \signop(\pdico_{I_n}^\star y_n) \neq \sigma_n(I_n)\}. \label{threshfails}
\end{align}
From \cite{sc14b}, e.g. proof of Proposition 7, we know that the right hand side in \eqref{threshfails} is in turn contained in the event $\mathcal E_n \cup \mathcal F_n$, where
\begin{align}
\qquad \mathcal E_n&:=\Big\{ y_n: \exists k \mbox{ s.t. } \Big| \sum_{j \neq k} \sigma_n(j) c_n\big(p_n(j)\big) \ip{\atom_j}{\atom_k}\Big| \geq u_1 \mbox{ or } |\ip{r_n}{\atom_k}| \geq  u_2 \Big\} \label{eventEdef}\\
\mathcal F_n&:= \Big\{y_n:  \exists k \mbox{ s.t. } \omega_k\Big|\sum_j  \sigma_n(j) c_n\big(p_n(j)\big) \ip{\atom_j}{ z_k}\Big| \geq u_3 \mbox{ or } \omega_k |\ip{r_n}{z_k}| \geq u_4 \Big\}\label{eventFdef}\\
&\mbox{for}\quad 2(u_1 + u_2 + u_3 + u_4) \leq c_n(S)\left(1-\frac{\eps^2}{2}\right)-c_n(S+1).
\end{align}
In particular if we choose $u_1=u_2=(c_n(S)-c_n(S+1))/7$, $u_3=u_1-\frac{\eps^2c_n(S)}{6}$ and $u_4=u_3/2$ we get that $\mathcal E_n $, which contains the event that thresholding using the generating dictionary $\dico$ fails, is independent of $\pdico$. 
To estimate the number of signals for which the thresholding summand is different from the oracle summand, it suffices to count how often $y_n \in \mathcal E_n$ or $y_n \in \mathcal F_n$,
\begin{align}
\sharp \{n : \signop(\ip{\patom_k}{y_n}) \, \chi(I_{\pdico,n}^t, k)\neq  \sigma_n(k) \, \chi(I_n,k) \} \leq \sharp \{ n: y_n \in  \mathcal E_n\} +\sharp \{ n: y_n \in \mathcal F_n\}.
\end{align} 
Substituting these bounds into \eqref{itksmsplit} we get,
\begin{align}
\left\| \bar \patom_k - \frac{C_r\gamma_{1,S}}{K}\atom_k\right\|_2 &\leq \frac{2 \sqrt{B+1}}{N}\sharp \{ n: y_n \in  \mathcal E_n\} + \frac{2 \sqrt{B+1}}{N} \sharp \{ n: y_n \in \mathcal F_n\}\notag\\
& \hspace{5cm}+ \left\|\frac{1}{N} \sum_n y_n \, \sigma_n(k) \, \chi(I_n,k) -  \frac{C_r\gamma_{1,S}}{K}\atom_k\right\|_2. \label{itksmsplit2}
\end{align}
If we want the error between $\bar\patom_k/\|\bar\patom_k\|_2$ and $\atom_k$ to be of the order $\factor \eps$, we need to ensure that the right hand side of \eqref{itksmsplit2} is less than $\factor \eps \cdot \frac{C_r\gamma_{1,S}}{K} $. \\
From Lemma~\ref{lemma1a} in the appendix we know that
\begin{align}
\P\left(\sharp \{ n: y_n \in  \mathcal E_n\} \geq  \frac{C_r\gamma_{1,S}N}{ 2K\sqrt{B+1}}\cdot(\eps_{\mu,\nsigma} + t_1) \right)\leq \exp\left( \frac{-t_1^2C_r\gamma_{1,S} N}{2K\sqrt{B+1}\,(2\eps_{\mu, \nsigma} + t_1) }\right).
\end{align}
Next Lemma~\ref{lemma1b} tells us that 
\begin{align}
&\P\left(\sharp \{ n: y_n \in  \mathcal F_n\} \geq  \frac{C_r\gamma_{1,S}N}{ 2K\sqrt{B+1}}\cdot(\tau \eps + t_2) \right)\leq \exp\left( \frac{-t_2^2C_r\gamma_{1,S} N}{2K\sqrt{B+1}\,(2\tau \eps + t_2) }\right),
\end{align}
whenever 
\begin{align}
\eps \leq \frac{\Delta_S}{\sqrt{98B}\left(\frac{1}{4} +\sqrt{\log\left(\frac{106K^2(B+1)}{\Delta_S C_r \gamma_{1,S}\tau}\right)}\right)}. \label{epsmax}
\end{align}
Finally by Lemma~\ref{lemma2} we have 
\begin{align}
\P\left( \left\| \frac{1}{N}\sum_n \frac{ \dico x_{c_n, p_n,\sigma_n} +\noise_n}{\sqrt{1+\|\noise_n \|_2^2}}\cdot \sigma_n(k) \cdot \chi(I_n,k) - \frac{C_r \gamma_{1,S}}{K} \atom_k \right\|_2 \geq t_3  \frac{C_r\gamma_{1,S}}{K} \right)\leq \exp\left(\frac{-t_3^2C^2_r\gamma_{1,S}^2N}{8 SK}+\frac{1}{4}\right),
\end{align}
whenever $0\leq t_3 \leq \frac{\sqrt{S}}{\sqrt{B}+2}$. 
Thus with high probability we have,
\begin{align}
\left\| \bar \patom_k - \frac{C_r\gamma_{1,S}}{K}\atom_k\right\|_2 &\leq \frac{C_r\gamma_{1,S}}{K} \left(\eps_{\mu,\nsigma} + t_1 + \tau \eps + t_2 + t_3\right).
\end{align}
To be more precise if we choose a target error $\epstarget \geq 4\eps_{\mu,\nsigma}$ and set $t_1=\epstarget/10$, $t_2= \max\{\epstarget, \eps\}/10$, $\tau=1/10$ and $t_3=\epstarget/5$, then except with probability
\begin{align}\label{prob_itksm_1step}
\exp\left( \frac{-C_r\gamma_{1,S} N\epstarget}{120 K\sqrt{B+1} }\right) + \exp\left( \frac{-C_r\gamma_{1,S} N \max\{\epstarget, \eps\} }{60 K\sqrt{B+1} }\right) + 2K\exp\left(\frac{-C^2_r\gamma_{1,S}^2N\epstarget^2}{200 SK}\right)
\end{align}
we have 
\begin{align}
\max_k \left\| \bar \patom_k - \frac{C_r\gamma_{1,S}}{K}\atom_k\right\|_2 &\leq \frac{C_r\gamma_{1,S}}{K} \cdot \frac{3}{4} \cdot\max\{\epstarget,\eps\}.
\end{align}
By Lemma~\ref{lemma_rescale} this further implies that
\begin{align}
d(\bar \pdico, \dico)=\max_k \left\| \frac{\bar \patom_k}{\|\bar \patom_k\|_2} - \atom_k\right\|_2 &\leq 0.83 \max\{\epstarget,\eps\}.
\end{align}
Note that in case of outliers we first have to split the sum in \eqref{itksm_batchsum} into the outliers, whose number concentrates around $N$ the probability of being an outlier, and the inliers for which we can use the same procedure as above, see \cite{bagrje14} for more details. Similarly the small portion of coefficients without (sufficiently) large gap can be included in the small number of signals for which thresholding fails.\\
\emph{Step 2:}
From Step 1 we know that in each iteration the error will either be decreased by at least a factor $0.83$ or if its already below $\epstarget$ will stay below $\epstarget$. So after $L$ iterations each using a new batch of $N$ signals, $d(\tilde \pdico,\dico) \leq \max\{\epstarget, 0.83^L d(\pdico,\dico)\} \leq \max\{\epstarget, 0.83^L\}$, except
with probability 
\begin{align}
L\left( \exp\left( \frac{-C_r\gamma_{1,S} N\epstarget}{120 K\sqrt{B+1}}\right) + \exp\left( \frac{-C_r\gamma_{1,S} N \max\{\epstarget, \eps\} }{60 K\sqrt{B+1} }\right) + 2K\exp\left(\frac{-C^2_r\gamma_{1,S}^2N\epstarget^2}{200 SK}\right)\right)
\end{align}
Setting $L = 6 \lceil \log( \epstarget^{-1})\rceil$ and taking into account that the failure probability of each iteration is bounded by $3K\exp\left(\frac{-C^2_r\gamma_{1,S}^2N\epstarget^2}{200 SK}\right)$ leads to the final estimate.\\
One
\end{proof}

For most desired precisions Theorem~\ref{th:itksm}, which is valid for a quite large hyper-cube of input dictionaries and a wide range of sparsity levels, will actually be sufficient. However, for completeness we specialise the theorem above to the case of strongly S-sparse, noiseless signals and show that in this case ITKsM can achieve arbitrarily small errors, provided enough samples.
\begin{corollary} \label{th:itksm_exact}
Let $\dico$ be a unit norm frame with frame constants $A\leq B$ and coherence $\mu$ and assume that the training signals $y_n$ are generated according to the signal model in \eqref{noisymodel2} with $r=0$ and 
coefficients that are strongly $S$-sparse with relative gap $\Delta_S > 2\mu S$. Fix a target error $\epstarget\geq 0$.
If for the input dictionary $\pdico$ we have
\begin{align}
d(\pdico, \dico)\leq \frac{\Delta _S- 2\mu S}{\sqrt{98B}\left(\frac{1}{4} +\sqrt{\log\left(\frac{1060K^2B}{(\Delta_S - 2\mu S)\gamma_{1,S}}\right)}\right)},
\end{align}
then after $6 \lceil \log( \epstarget^{-1})\rceil$ iterations of ITKsM, each on a fresh batch of $N$ training signals, the output dictionary $\tilde \pdico$ satisfies 
\begin{align}
d( \tilde \pdico, \dico) &\leq \epstarget 
\end{align}
except with probability 
\begin{align}
18 \log( \epstarget^{-1})K\exp\left(\frac{-\gamma_{1,S}^2N\epstarget^2}{200 SK}\right).
\notag
\end{align} 
\end{corollary}

\noindent The proof is analogue to the one of Theorem~\ref{th:itksm} and can be found in Appendix~\ref{proof_itksm_exact}.\\

Let us again discuss the result. The main difference to Theorem~\ref{th:itksm} is that the condition $\Delta_S \geq 2\mu S$ can only hold for much lower sparsity levels, that is $S = O(\mu^{-1})$ and thus for incoherent dictionaries up to the square root of the ambient dimension $O(\sqrt{d}) \ll O(d/\log K)$. It is also no surprise that once the input dictionary is up to a log factor within this radius, ITKsM can achieve arbitrarily small errors. Indeed once $\Delta_S \geq 2\mu S$ thresholding is always guaranteed to recover the sparse support of a signal given the ground truth dictionary or a slight perturbation of it, \cite{scva07}.\\
To again turn the corollary into something less technical and more interesting we combine it with the corresponding theorem.
If the coefficients are strongly $S$-sparse the minimally achievable error using Theorem~\ref{th:itksm} will be smaller than 
the error we need for Corollary~\ref{th:itksm_exact} to take over and so we get the following O notation result.

\begin{ocorollary}[\bf \ref{th:itksm_exact}]Assume that in each iteration the number of noiseless, exactly S-sparse training signals scales as $O(K\log K \epstarget^{-2})$. If $S\leq O(\mu^{-1})$ then with high probability for any starting dictionary $\pdico$ within distance $\eps\leq O(1/\sqrt{\log K})$ to the generating dictionary after $O( \log( \epstarget^{-1}))$ iterations of ITKsM, each on a fresh batch of training signals, the distance of the output dictionary $\tilde \pdico$ to the generating dictionary will be smaller than $\epstarget$.
\end{ocorollary}

While a convergence radius of around $1/\sqrt{\log K}$, admissible sparsity levels up to $d/\log K$ and a dependence of the sample complexity on only $K\log K$ is very positive, the dependence of the sample complexity on the squared inverse target error $\epstarget^{-2}$ for noiseless exactly S-sparse signals is somewhat disappointing. Again note that in the case of noisy signals information theoretic arguments indicate that this factor is unavoidable, \cite{juelgo14}. Looking at the proof of Theorem~\ref{th:itksm} we see that the reason for this factor is the slow concentration of the sums $\frac{1}{N} \sum_n y_n \, \sigma_n(k) \, \chi(I_n,k)$ around the atom $\atom_k$. This can in turn be explained by the fact that via the summation we have to cancel out the equally sized contribution of all other atoms. Actively trying to cancel out these contributions already before the summation, that is summing residuals instead of signals, should therefore accelerate the concentration, and lead to a lower sample complexity in case of noiseless signals and better constants in case of noisy signals.
We will concretise these ideas in the next section.

\section{Dictionary Learning via ITKrM}\label{sec:itkrm}

There are several ways to remove the contribution of all atoms in the current support $I_{\pdico,n}^t$ except for $\patom_k$.
The maybe most obvious way is to consider $Q(\pdico_{I_{\pdico,n}^t\setminus k})y_n=[\I_d -P(\pdico_{I_{\pdico,n}^t\setminus k})]y_n $. Unfortunately this residual has several 
disadvantages, the most severe being that it is not clear whether for the oracle supports and oracle signs the corresponding sum of residuals concentrates around a multiple of the atom $\atom_k$,
\begin{align}
\E \left(\frac{1}{N}\sum_{n} Q(\dico_{I_{n}\setminus k})y_n \cdot \sigma_n(k) \cdot  \chi(I_{n}, k)\right)  \propto \E_{I:k\in I}\left( Q(\dico_{I\setminus k})\, \atom_k \right) \stackrel{ ?}{ \propto}  \atom_k.
\end{align}
We suspect that equality can only hold for tight dictionaries and that an additional constraint such as minimal incoherence is needed. We therefore choose a perhaps less obvious but more stable residual $a_{n,k} (\pdico) = y_n - P(\pdico_{I_{\pdico,n}^t}) y_n + P(\patom_k) y_n$, which captures the contribution of the current atom $\atom_k$ as well as the approximation error in $\pdico$, that is $y_n - P(\pdico_{I_{\pdico,n}^t}) y_n $. Replacing the signal means in ITKsM with residual means we arrive at the new algorithm, iterative thresholding and K residual means (ITKrM).

\begin{algorithm}[ITKrM one iteration] 
Given an input dictionary $\pdico$ and $N$ training signals $y_n$ do:
\begin{itemize}
\item For all $n$ find $I_{\pdico,n}^t= \arg\max_{I: | I |=S} \| \pdico_I^\star y_n\|_1$.
\item For all $k$ calculate
\begin{align}
\patomn_k= \frac{1}{N} \sum_{n} \big[y_n - P&(\pdico_{I_{\pdico,n}^t}) y_n + P(\patom_k) y_n\big] \cdot \signop(\ip{\patom_k}{y_n}) \cdot  \chi(I_{\pdico,n}^t, k).
\end{align}
\item Output $\ppdico=(\patomn_1/\|\patomn_1\|_2, \ldots, \patomn_K/\|\patomn_K\|_2)$.
\end{itemize}
\end{algorithm}
Again 
ITKrM inherits most computational properties of ITKsM. As such it can again be stopped after a fixed number of iterations or once a stopping criterion, such as the improvement below some threshold, is reached. Only one signal has to be processed at a time, making it suitable for an online version and parallelisation. Its computational complexity is slightly larger than for ITKsM because of the projections $P(\pdico_{I_{\pdico,n}^t}) y_n$.
If computed with maximal numerical stability, these have an overall cost of $O(S^2 dN)$, which corresponds to the QR decompositions of $\pdico_{I^s_n}$. However, since the achievable precision in the learning is usually limited by the number of available training signals rather than the numerical precision, it is computationally more efficient to precompute the gram matrix $\pdico^\star \pdico$ and calculate the projections via the eigenvalue decompositions of $\pdico_{I^s_n}^\star \pdico_{I^s_n}$, which is less stable but reduces the overall cost to $O(S^3N)$. Still for $S\geq d^{2/3}$ these computations become the determining factor; we will see that $S$ can again be of the order $O(\mu^{-2}/\log K)\approx O(d/\log K)$.
In the next subsection we will analyse which convergence properties of ITKsM translate to ITKrM.

\subsection{Convergence Analysis of ITKrM}
As for ITKsM we focus on the more realistic case of non exactly S-sparse
and/or relatively noisy signals and specialise our results to exactly S-sparse, noiseless signals and moreover the case where $S\leq O(\mu^{-1})$ later.

\begin{theorem} \label{th:itkrm}
Let $\dico$ be a unit norm frame with frame constants $A\leq B$ and coherence $\mu$ and assume that the training signals $y_n$ are generated according to the signal model in \eqref{noisymodel2} with coefficients that are $S$-sparse with absolute gap $\beta_S$ and relative gap $\Delta_S$. Assume further that $S\leq \frac{K}{98B}$ and $\eps_\delta:=K  \exp\left(-\frac{1}{4741\mu^2 S} \right)\leq \frac{1}{48(B+1)}$.\\
Fix a target error $\epstarget \geq 8 \eps_{\mu, \nsigma}$, with 
\begin{align}
\eps_{\mu, \nsigma} = \frac{8K^2\sqrt{B+1}}{C_r \gamma_{1,S}} \exp\left(\frac{-\beta_S^2}{98\max\{ \mu^2,\nsigma^2\}}\right),
\end{align}
compare~\eqref{epsmin}, and assume that $\epstarget \leq 1-\gamma_{2,S} + d\nsigma^2$.\\
If for the input dictionary $\pdico$ we have 
\begin{align} \label{epsmax_r}
d(\pdico, \dico)\leq \frac{\Delta_S}{ \sqrt{98B}\left(\frac{1}{4} +\sqrt{\log\left(\frac{2544K^2(B+1)}{\Delta_S C_r \gamma_{1,S}}\right)}\right)} \qquad \mbox{and} \qquad d(\pdico, \dico)\leq \frac{1}{32\sqrt{S}},
\end{align}
then after $12 \lceil \log( \epstarget^{-1})\rceil$ iterations of ITKrM each on a fresh batch of $N$ training signals the output dictionary $\tilde \pdico$ satisfies 
\begin{align}
d( \tilde \pdico, \dico) &\leq \epstarget 
\end{align}
except with probability 
\begin{align} \label{itkrm_failure}
60\lceil \log( \epstarget^{-1})\rceil K\exp\left(\frac{-C^2_r\gamma_{1,S}^2N\epstarget^2}{576 K\max\{S,B+1\}\left(\epstarget+1 -\gamma_{2,S} + d\nsigma^2\right)}\right). 
\end{align}
\end{theorem}

\begin{proof}The proof follows the same two step procedure as the proof of Theorem~\ref{th:itksm}, where in the first step we prove that one iteration will reduce the error by a factor $\kappa<1$ with high probability and then iterate this result. To prove the first step we again use a triangular inequality argument. So we check how often thresholding with $\pdico$ fails. Assuming thresholding recovers the generating support we show that the difference between the oracle residual (based on the generating sign and support) using
$\dico$ and the oracle residual using $\pdico$ concentrates around its expectation, which is small. Finally we show that the sum of residuals using $\dico$ converges to a scaled version of $\atom_k$. To keep the flow of the paper we do not give the full proof here but in Appendix~\ref{proof_itkrm}.
\end{proof}

Let us discuss the result. First we see that compared to the corresponding theorem for ITKsM we need somewhat more conditions. The first two extra conditions on the sparsity level $S\leq \frac{K}{98B}$ and $48(B+1)\eps_\delta<1$ are technicalities. For all but the most ideal cases they are already implied by having a limiting error $\eps_{\mu,\nsigma}$ smaller than one. Since $\beta_S\leq 1/\sqrt{S}$ the first condition is implied as soon as $\mu^2$ is larger than $B/K$, where at best we have $\mu^2=\frac{B-1}{K-1}$. The second condition is a substitute for having small isometry constant of the generating dictionary $\delta_S\leq \frac{1}{4}$ and guarantees that most support sets of size $S$ have $\delta_I(\dico)\leq \frac{1}{4}$. It is implied by $\eps_{\mu,\nsigma}\leq 1$ as soon as $\beta_S$ is smaller than $\frac{1}{7\sqrt{S}}$ or equivalently the dynamic range of the coefficients is larger than 7. \\
The target error can again be chosen closer to the limiting error at the cost of horrible constants. Also note that the condition that the target error should be smaller than the expected squared approximation error and noise is again a technicality. If both noise and approximation error are so small that a larger target error makes sense we get the same result but with a smaller failure probability. To get an idea how such a result would look like we refer the reader to the corollary below, where we assume exactly sparse noiseless signals. \\
The only extra condition that changes the quality of the result is the second condition on the convergence radius. Assuming that $\Delta_S = O(1)$ the first bound in \eqref{epsmax_r} is of the order $O(1/\sqrt{\log K})$, so as soon as $S \geq \log K$, meaning for most practically relevant cases, the second bound will be more restrictive. This decreased convergence radius of ITKrM compared to ITKsM is a little disappointing 
but seems unavoidable. The reason for this is that the expected difference between the oracle residuals using $\pdico$ and $\dico$ depends on the operator norms of the rescaled perturbation matrices $\| B_I\|_{2,2}$, compare Lemma~\ref{lemma4}. If the perturbation dictionary is quasi constant, that is before normalisation $z_k = v - P(\atom_k) v$ for some $v\neq 0$,
then $\|B_I\|_{2,2} \approx  \sqrt{S} \eps$ for all possible subsets $I$, so we need $\eps \leq 1/\sqrt{S}$.\\
The advantage over ITKrM is that for low expected noise levels and approximation errors, $1 -\gamma_{2,S} + d\nsigma^2\ll 1$, we get better constants in the sample complexity. Actually from the probability bound in \eqref{itkrm_failure} we can already guess that for exactly sparse, noiseless signals we can reduce the factor $\eps^{-2}$ in the exponent to $\eps^{-1}$. Before specialising the theorem to noiseless signals we again provide a qualitative result, which combines the theorem above with the corresponding theorem for ITKsM in order to deal with the reduced convergence radius. That is we first exploit the large convergence radius of ITKsM and run ITKsM to arrive at an error $O(1/\sqrt{S})$. Then we exploit the lower sample complexity of ITKrM to arrive at the target error.
\begin{otheorem}[\bf \ref{th:itkrm}] Assume that in each iteration the number of training samples $N$ scales as $O(K\log K \epstarget^{-2})$. If $S\leq \frac{1}{\mu^2 \ell \log K}$ then with high probability for any starting dictionary $\pdico$ within distance $\eps\leq O(1/\sqrt{\log K})$ to the generating dictionary after $O(\log(S))$ iterations of ITKsM and $O( \log( \epstarget^{-1}))$ iterations of ITKrM the distance of the output dictionary $\tilde \pdico$ to the generating dictionary will be smaller than 
\begin{align}
\max \left\{\epstarget, O\left(K^{2-\ell}\right)\right\}.
\end{align}
\end{otheorem}
Unfortunately the better constant in the sample complexity of ITKrM disappears in the O notation and we cannot really see the improvement over ITKsM. We therefore specialise again to noiseless, strongly S-sparse signals.

\begin{corollary}\label{th:itkrm_exact}
Let $\dico$ be a unit norm frame with frame constants $A\leq B$ and coherence $\mu$ and assume that the training signals $y_n$ are generated according to the signal model in \eqref{noisymodel2} with $r=0$ and coefficients that are exactly and strongly $S$-sparse with relative gap $\Delta_S > 2 \mu S$.
Fix a target precision $\epstarget>0$.
If for the input dictionary $\pdico$ we have $d(\pdico, \dico)\leq \frac{1}{32\sqrt{S}}$ and 
\begin{align} \label{epsmax_rnoiseless}
d(\pdico, \dico)\leq \frac{\Delta_S - 2\mu S}{\sqrt{12}\left(\frac{1}{4} +\sqrt{\log\left(\frac{23 K^2\sqrt{B}}{(\Delta_S - 2\mu S) \gamma_{1,S}}\right)}\right)} \qquad \mbox{and} \qquad d(\pdico, \dico)\leq \frac{1}{32\sqrt{S}},
\end{align}
then after $9 \lceil \log( \epstarget^{-1})\rceil$ iterations of ITKrM, each on a fresh batch of $N$ training signals, the output dictionary $\tilde \pdico$ satisfies 
\begin{align}
d( \tilde \pdico, \dico) &\leq \epstarget \notag
\end{align}
except with probability 
\begin{align}
 \quad 27 K \lceil \log( \epstarget^{-1})\rceil \exp\left( \frac{-\gamma^2_{1,S} N \epstarget}{144\,K\max\{S,B\} }\right).
\end{align}
\end{corollary}
The proof sketch can be found in the Appendix~\ref{proof_itkrm_exact}.\\

The above corollary clearly reveals the influence of the underlying signal model on dictionary learning results. So assuming that the signals are noiseless and exactly sparse and that $S$ is only of the order $O(\mu^{-1})=O(\sqrt{d})$, we get that one iteration of ITKrM will reduce the error as long as the number of samples scales as $O(K \eps^{-1})$, meaning the influence of the target error is reduced by a factor $\eps^{-1}$!\\
Again combining with ITKsM and assuming that the stronger restriction on the convergence radius is the second bound in \eqref{epsmax_rnoiseless}, we get the following quantitative results.
\begin{ocorollary}[\bf\ref{th:itkrm_exact}] Assume that in each iteration the number of noiseless, exactly S-sparse training signals scales as $O(K\log K \epstarget^{-1})$. If $S\leq O(\mu^{-1})$ then with high probability for any starting dictionary $\pdico$ within distance $\eps\leq O(1/\sqrt{\log K})$ to the generating dictionary after $O( \log(S))$ iterations of ITKsM and $O( \log( \epstarget^{-1}))$ iterations of ITKrM, each on a fresh batch of training signals, the distance of the output dictionary $\tilde \pdico$ to the generating dictionary will be smaller than $\epstarget$.
\end{ocorollary}
Before a final discussion of our results we first illustrate our theoretical findings with some numerical simulations, which give interesting insights into the average convergance radius of the algorithms and indicate that in practice ITKrM can be a very powerful low complexity alternative to K-SVD.

\section{Numerical Simulations}\label{sec:sim}
To complement our theoretical findings, we conduct two small numerical experiments both on synthetic and real data\footnote{A Matlab Swiss knife (mini-toolbox) for playing with ITKrM and reproducing the experiments can be found at \url{http://homepage.uibk.ac.at/~c7021041/ITKrM.zip}.}. First
we test the average case convergence radius and speed of the ITKsM and ITKrM algorithm, by 
running both algorithms on noiseless and noisy training data, using three different types of
initialisations with varying distance to the generating dictionary.\\
We generate our training signals based on the signal model in \eqref{noisymodel2}.
As generating dictionary $\dico$ we choose the dictionary
consisting of the Dirac basis and the first half of the elements of the 
discrete cosine transform basis in $\R^{d}$ with $d=256$, meaning
$K=3*d/2 = 384$, which has coherence $\mu = \sqrt{2/d} \approx 0.088$. 
Given a sparsity level $S$, to simulate noiseless, exactly sparse signals we choose $c$ with
$c_1=\ldots=c_S=1/\sqrt{S}$ and $c_k=0$ for $k>S$, meaning dynamic range 1. To simulate
noisy signals with a higher dynamic range we choose a decay parameter $c_b$ uniformely
at random in [0.9,1], and let the first $S$ entries be a geometric sequence, that is
$c_k=b_0 * c_b^k$ for $k\leq S$ and $c_k=0$ for $k>S$, where $b_0$ is a scaling parameter ensuring
that $\|c\|_2=1$. The noise $r$ is chosen as a centered Gaussian with variance $1/d$,
that is $r(k) \sim \mathcal N (0, 1/\sqrt{d})$, resulting in an expected signal to noise ratio of~1. 
The three different types of initialisations are created by first choosing vectors $z_k$ uniformly  
at random from the unit sphere in $\R^d$, and then setting $$\patom_k =\alpha \cdot \atom_k + \omega \cdot \frac{Q(\atom_k) z_k}{\| Q(\atom_k) z_k \|_2}$$ for the ratios
$\alpha:\omega=1:1$ and $\alpha:\omega=1:4$. We also consider the completely random initialisation 
$\patom_k = z_k$. \\
For each initialisation dictionary we then run 100 iterations of ITKsM and ITKrM with the true sparsity level and dictionary size as input parameters, each time using a new batch of 100000 noiseless, respectively noisy signals. Figure~\ref{fig1} shows the average convergence respectively recovery rates over 20 trials for the three types of initialisations, using noiseless or noisy signals and for sparsity levels $S=4,8,12,16$. 

\begin{figure}[hp]
\begin{tabular}{cc}
 \includegraphics[width=8cm]{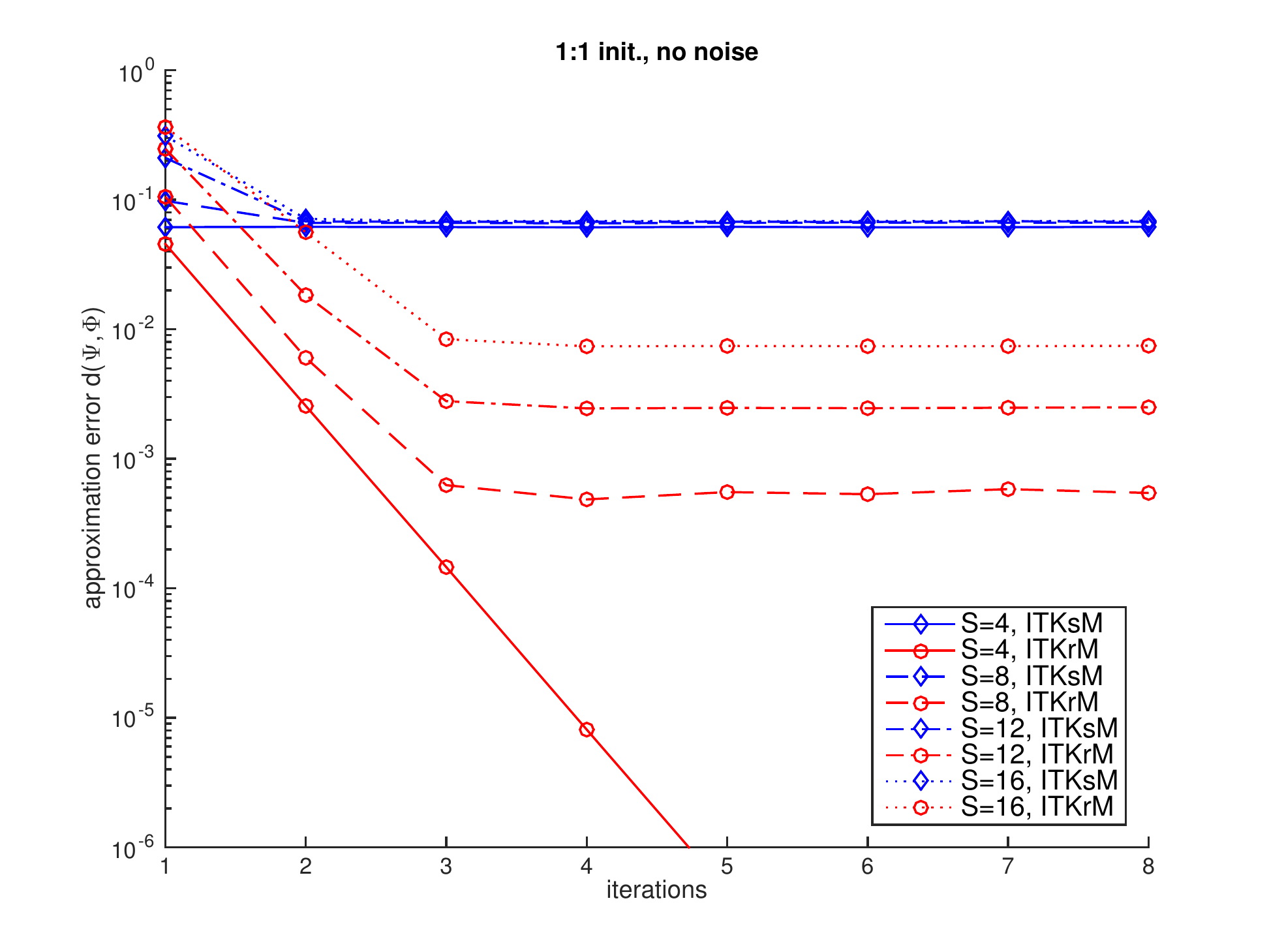} & \includegraphics[width=8cm]{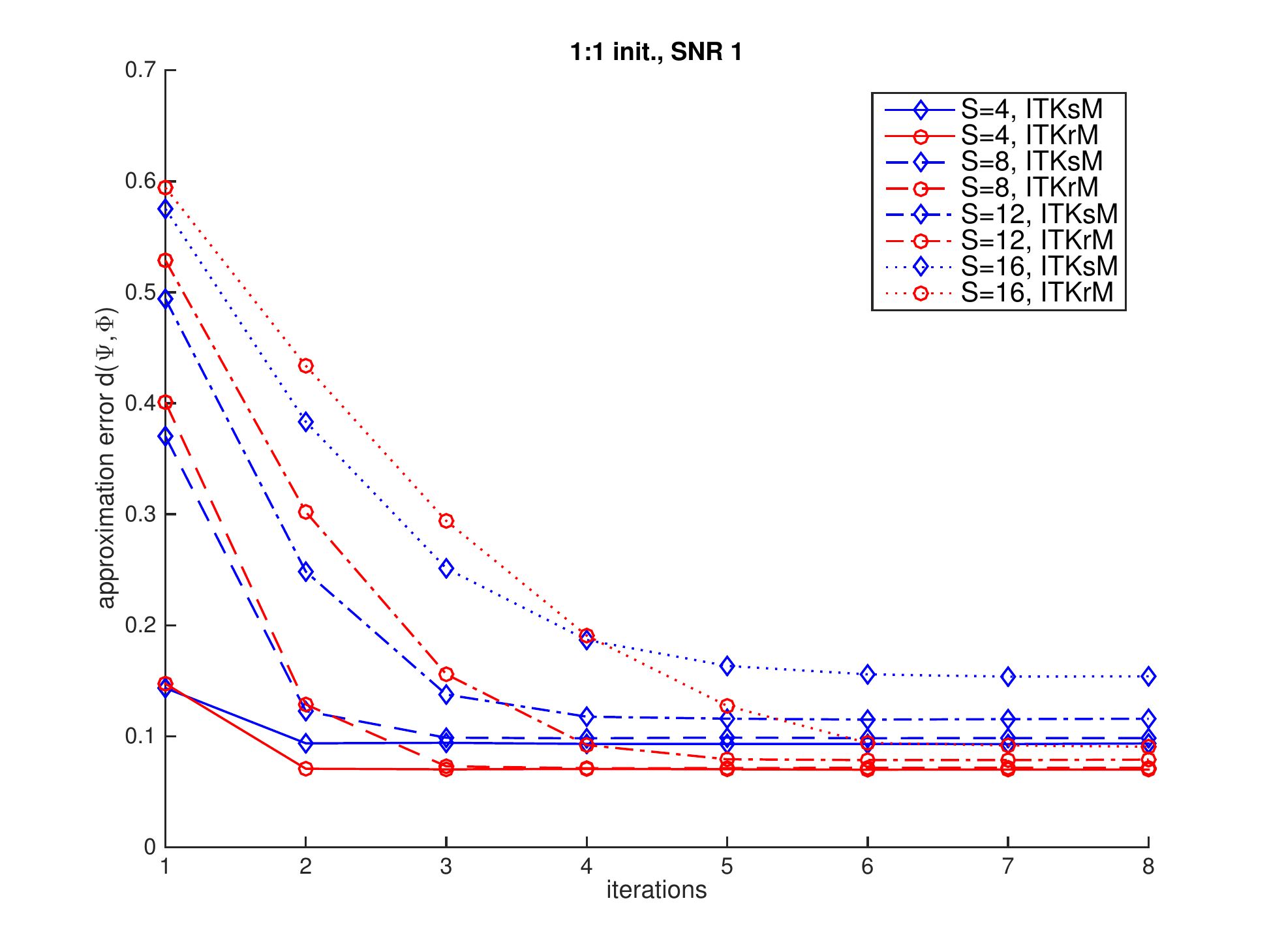}\\
  (a) & (b)\\
   \includegraphics[width=8cm]{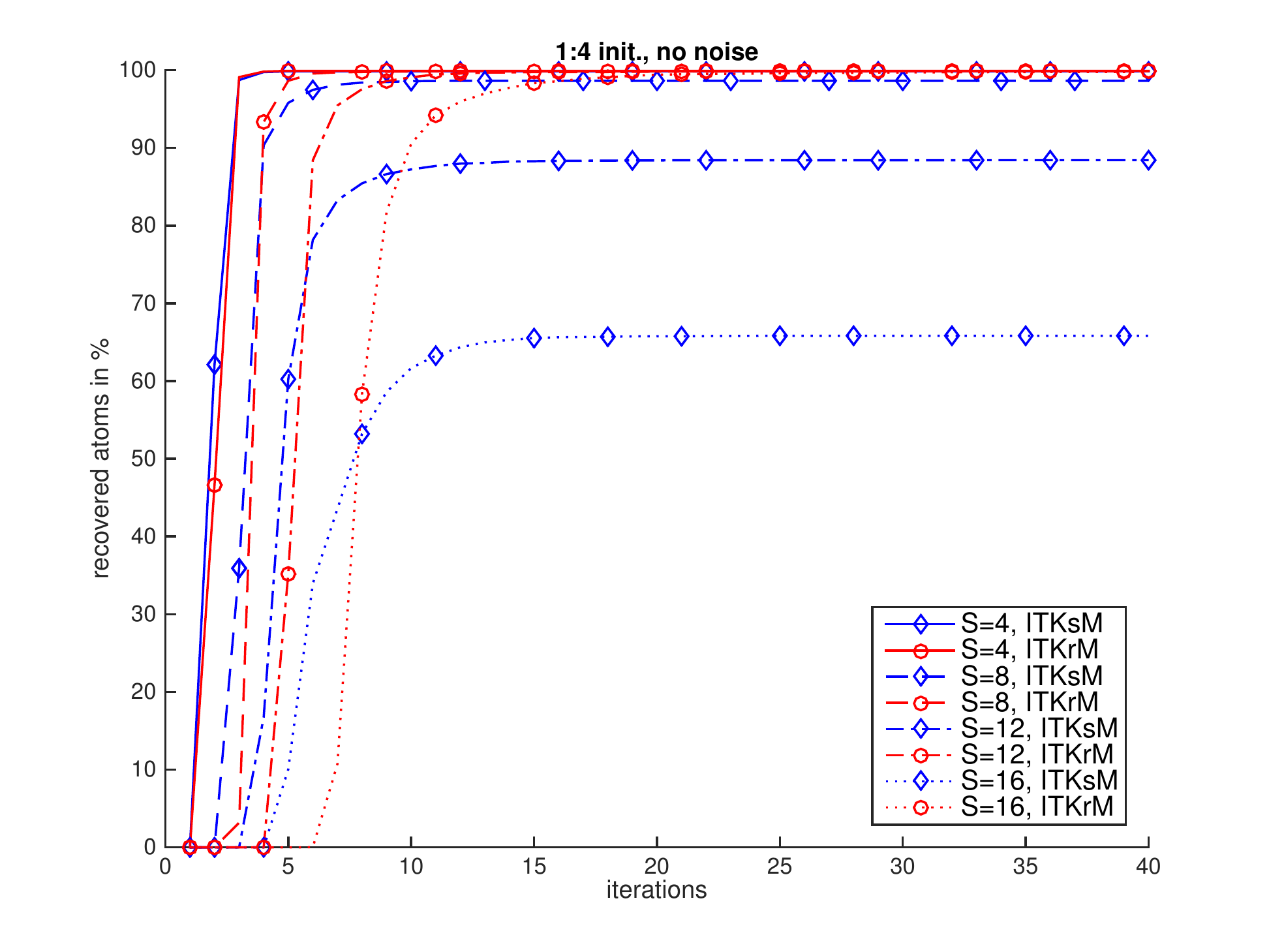} & \includegraphics[width=8cm]{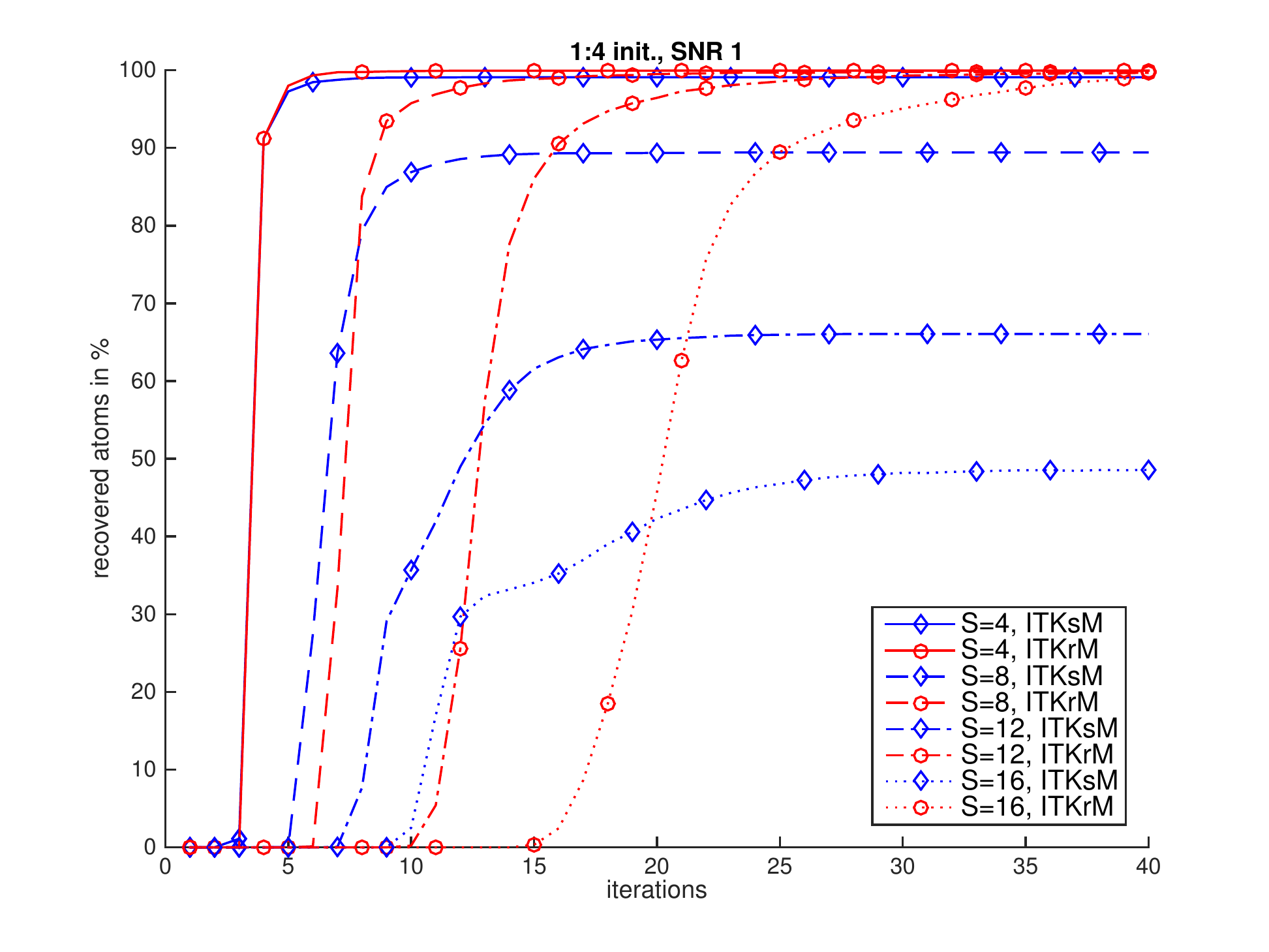}\\
  (c) & (d)\\
   \includegraphics[width=8cm]{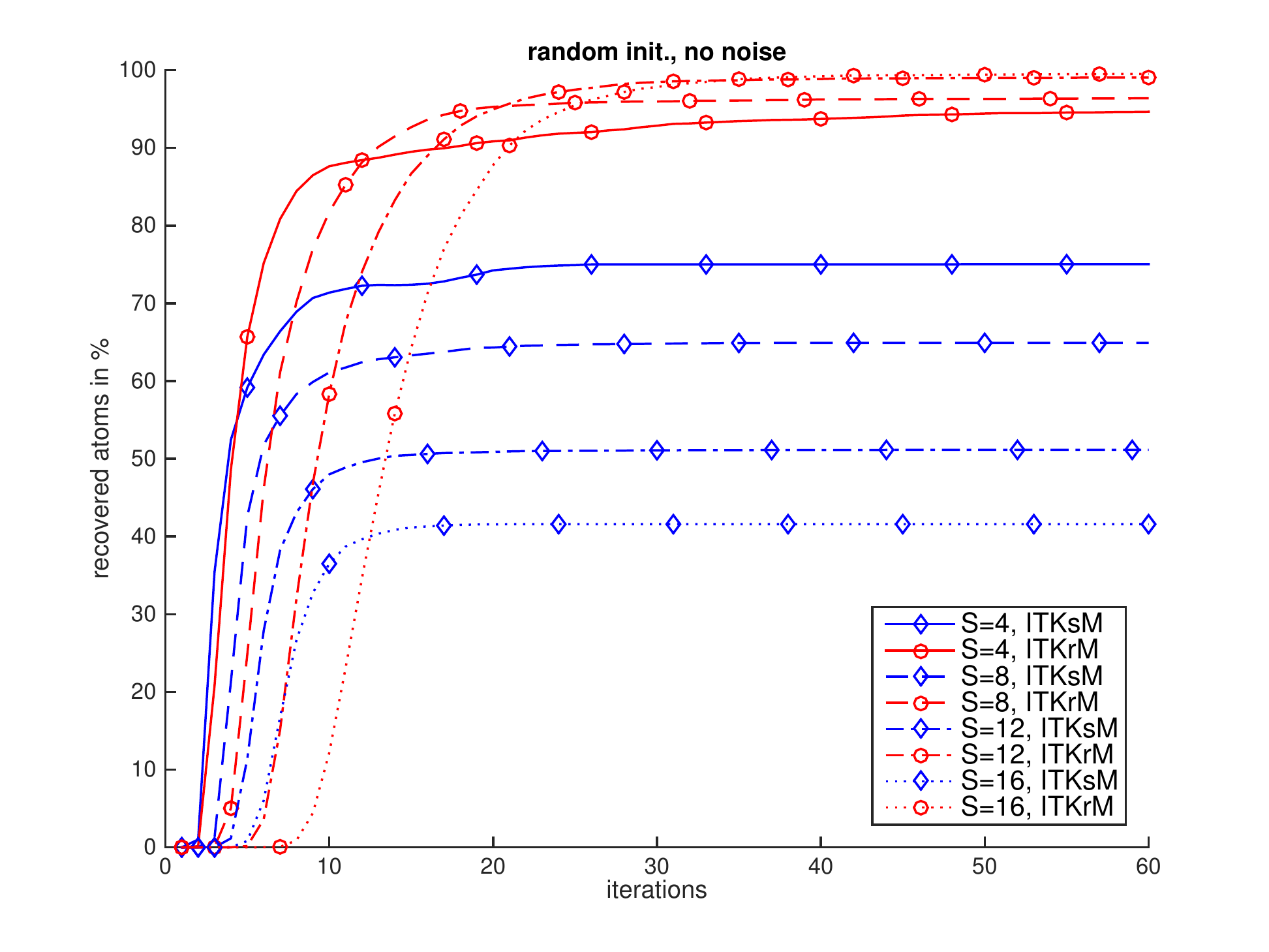} & \includegraphics[width=8cm]{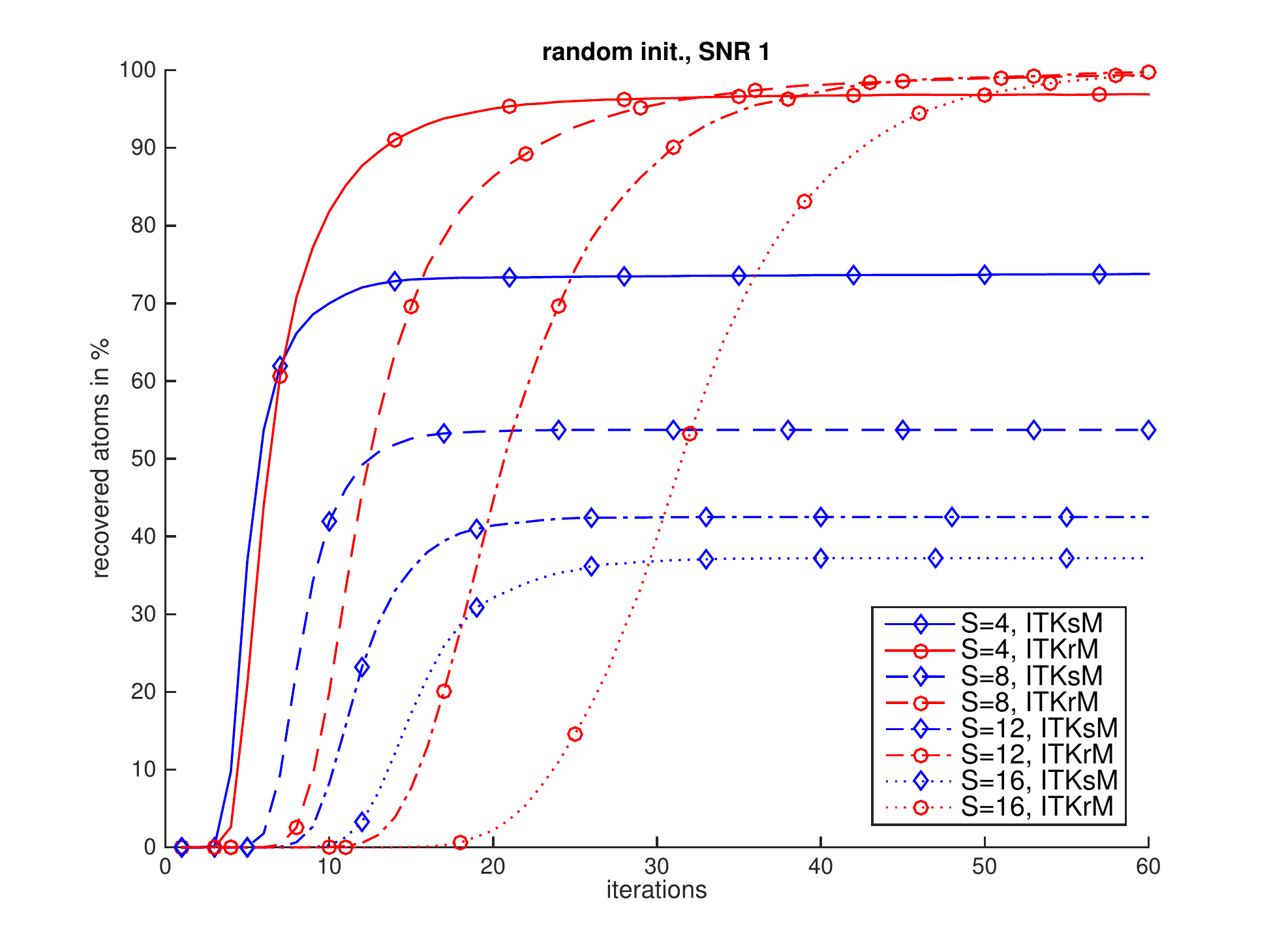}\\
  (e) & (f)\\
  \end{tabular}
    \caption{Convergence respectively recovery rates of ITKsM and ITKrM for three initialisation types, corresponding to increasing distance to the generating dictionary, and using training signals
    with varying sparsity levels both in the noiseless and noisy case. \label{fig1}}
\end{figure}
 
For the $1:1$ initialisations, despite the fact that the corresponding distance between the initialisations and the generating dictionary is much larger than the our estimated convergence radius, $d(\Psi,\Phi) = \sqrt{2-\sqrt{2}} \approx 0.7654 \gg 1/\sqrt{\log K}$, both algorithms always converge to the generating dictionary, so we plot the distance $d(\Psi^{(n)},\dico)$ between the generating dictionary $\dico$ and the output dictionary
of the n-th iteration $\Psi^{(n)}$, Figure~\ref{fig1}(a/b). As predicted by our theoretical results, using the same number of signals, ITKrM always leads to a more accurate estimate than ITKsM. As shown in Figure~\ref{fig1}(a) for the noiseless signals with dynamic range 1 this difference is quite pronounced and especially in the case $S=4\leq \mu^{-1}/2$, the regime of unique sparsity, the precision of ITKrM is limited rather by the machine precision rather then the amount of training signals. From Figure~\ref{fig1}(b) we see that both algorithms are locally stable even for the comparatively low signal to noise ratio $SNR=\E(\|r\|_2^2)/\E(\|\dico_x\|_2) = 1$ and coefficients with dynamic ranges varying between 1 and $0.9^{1-S}$.\\
For the $1:4$ initialisations, corresponding to distance $d(\Psi,\Phi) = \sqrt{2-2/\sqrt{17}} \approx 1.2308$ between the initialisations and the generating dictionary, we do not always have convergence to the generating dictionary.  We therefore plot the percentage of atoms recovered by the algorithm, using the convention that an atom $\atom_k$ is recovered if $\max_\ell |\ip{\patom_\ell^{(n)}}{\atom_k}| \geq 0.99$, compare \cite{ahelbr06}. Counterintuitively to our theoretical results ITKrM turns out to be much more stable to far away initialisations. As we can see from Figure~\ref{fig1}(c), in the case of noiseless signals ITKrM always recovers more than 99\% of the atoms, while the recovery rate of ITKsM deteriorates quite drastically as the sparsity parameter $S$ increases. To be more precise after 100 iterations ITKrM recovers the full dictionary for $17, 17, 15$ and $8$ out of 20 initialisations for $S$ taking values $4, 8, 12$ and $16$ respectively, while ITKsM can only recover the full dictionary in case $S=4$, (15 out of 20 initialisations), and for all other sparsity levels fails every time. The better performance of ITKrM is further confirmed by the results for noisy signals shown in Figure~\ref{fig1}(d). While the recovery rates of ITKsM deteriorate further and even for $S=4$ ITKsM can never recover the full dictionary, ITKrM continues to perform well. Indeed for ITKrM we can report a dithering effect, that is a better performance
in noisy conditions, as ITKrM recovers the full dictionary 18 out of 20 times for $S=4$ and always recovers the full dictionary for the other sparsity levels.\\
For the random initialisations we again plot the recovery rates, which confirm the trends observed for the $1:4$ initialisations, Figure~\ref{fig1}(e/f).
While the recovery rates of ITKsM are at best around 73\% in the noiseless case for $S=4$ decreasing to around $35\%$ in the noisy case for $S=16$, ITKrM always manages to recover at least 93\% of the atoms. Interestingly even though again recovery speed decreases as $S$ increases, both in the case of noiseless and noisy signals the recovery rates increase with $S$. So in the case of noiseless signals for $S=12$ and $S=16$ we again get a more than 99\% recovery rate and can even report one respectively two full recoveries. In the case of noisy signals the dithering effect is now clearly visible and remarkably in case $S=8$ ITKrM can recover the full dictionary 17 out of 20 times and for $S=12,16$ we always get full recovery.\\
Finally we also conduct a small experiment on image data to show that the more promising  ITKrM algorithm is not merely a pretty toy for synthetic set-ups but indeed useful in practice. In particular for two $256\times 256$ images, Fabio and Barbara, we take all $62 001$ possible $8\times 8$ patches, normalise them and afterwards subtract their mean, that is we project the patches onto the orthogonal complement of the constant atom $\atom_1\equiv 1/8$. On these patches we then learn a dictionary of 63 atoms, corresponding to the dimensionality of the signals after subtracting the mean (d=K=63). To be precise, we use a random initialisation, set the sparsity level $S=5$, and in each of the 100 iterations use 10000 randomly selected patches. Figure~\ref{fig2} shows the two images together with their respective learned dictionaries (including the constant atom~$\atom_1$).

\begin{figure}[t]
\begin{tabular}{cc}
 \includegraphics[height=6cm]{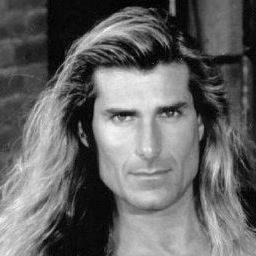} & \includegraphics[height=6cm]{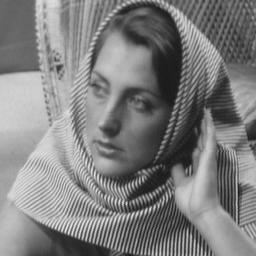}\\
  (a) & (b)\\
   \includegraphics[width=8cm]{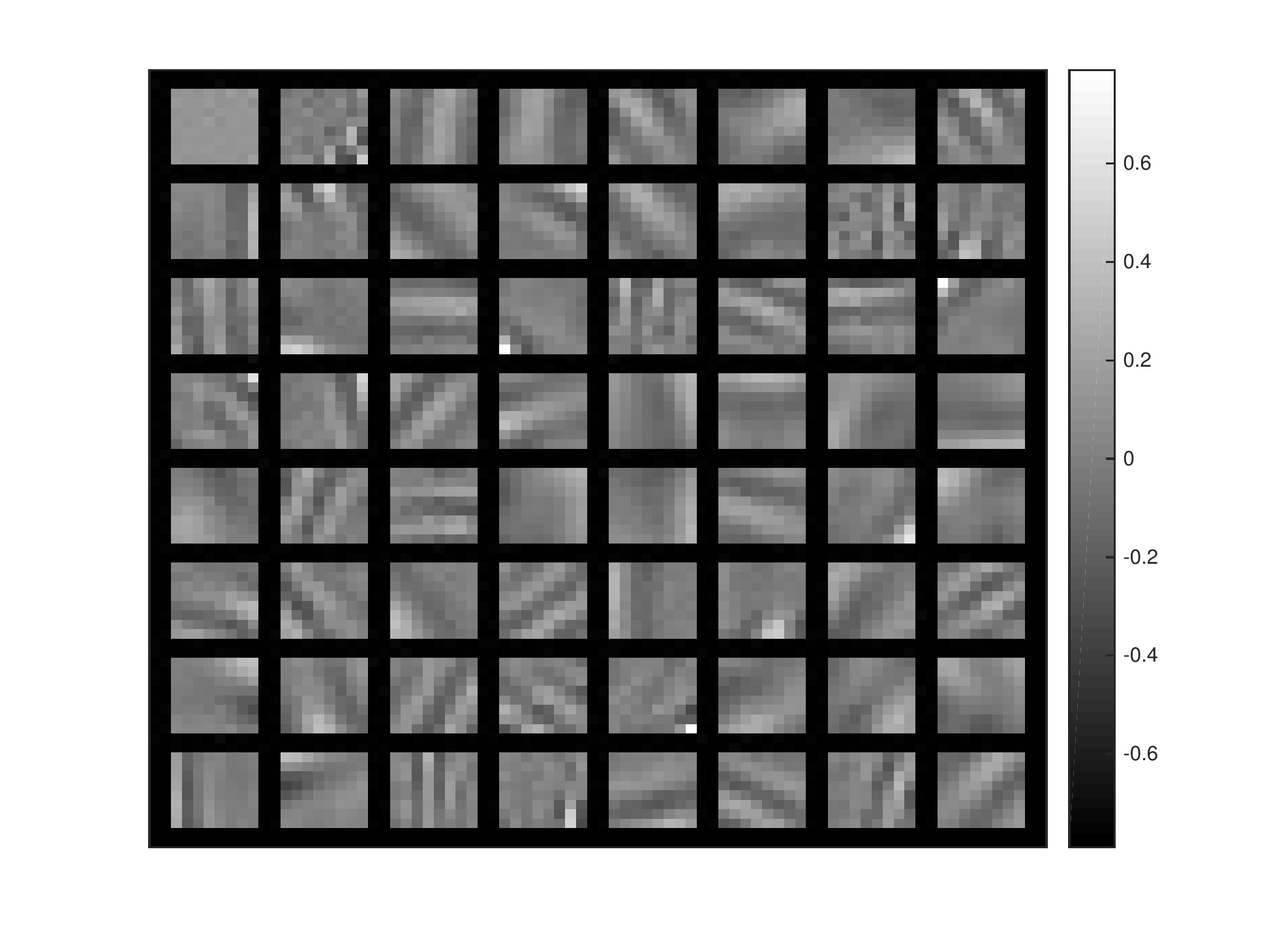} & \includegraphics[width=8cm]{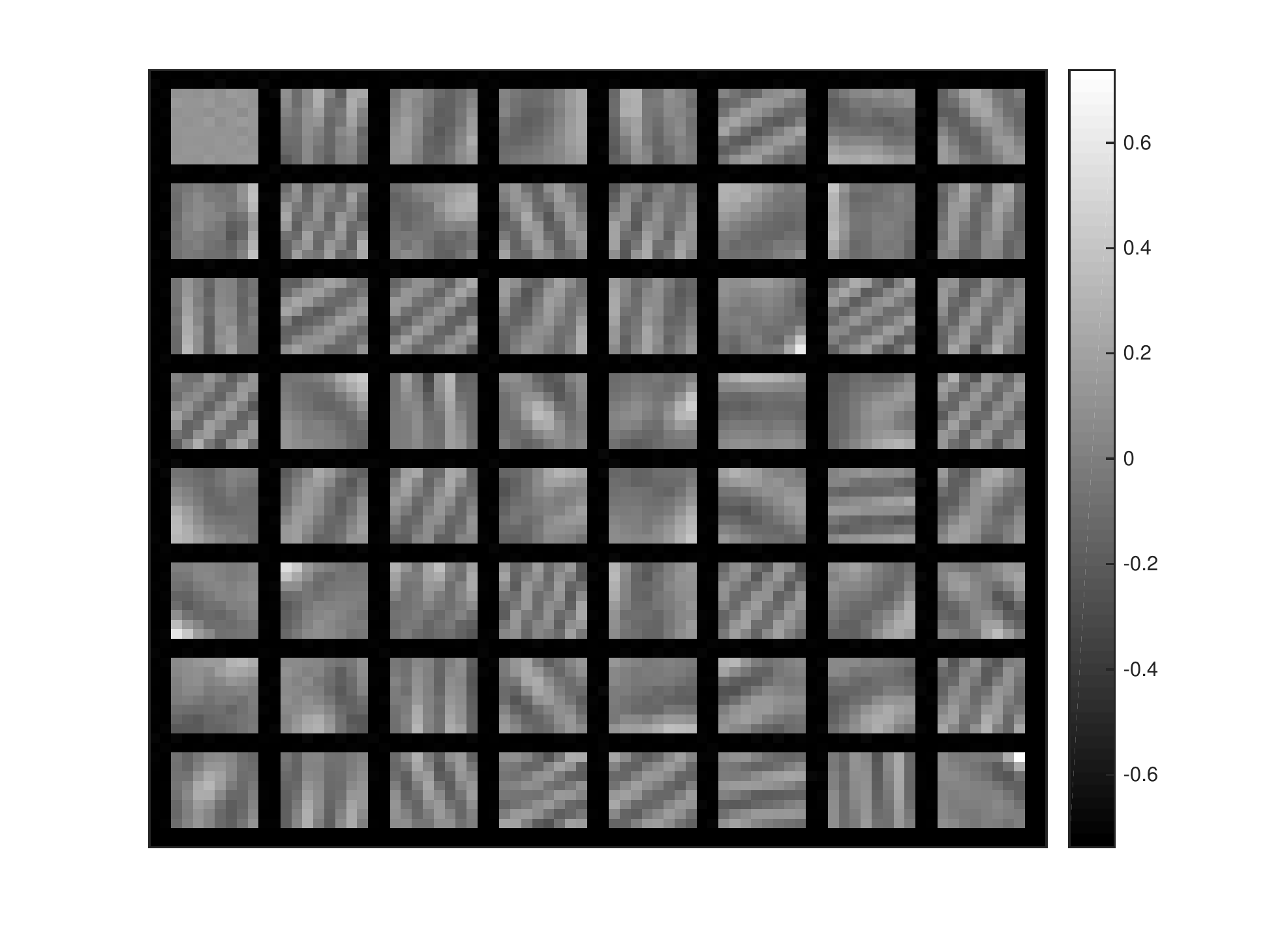}\\
  (c) & (d)\\
  \end{tabular}
    \caption{$256\times 256$ images together with the dictionaries learned on their $8\times8$ patches, (a,c) Fabio, (b,d) Barbara. \label{fig2}}
\end{figure}
 
As we can see ITKrM is able to calculate meaningful dictionaries also on real data. In particular observe that even though we have used the same initialisation the dictionary learned on Barbara contains a lot more high frequency wave-like atoms, which capture the texture of the scarf. For the sake of conciseness we do not go into more details about the approximation performance of the learned dictionaries or possible image processing applications here but refer the interested reader to \cite{ho16, sc16, nasc16}. 
Instead we now turn to a final discussion of our results.
 
\section{Discussion}\label{sec:discussion}

We have shown that iterative thresholding and K-means is a very attractive dictionary learning method, since it has low computational complexity, $O(dKN)$ omitting log factors, can be used in parallel or online, has convergence radius $O(1/\sqrt{\log K)}$ and sample complexity $O(K\log K \epstarget^{-2})$ for a target error $\epstarget$, which reduces to $O(K\log K \epstarget^{-1})$ in the case of noiseless exactly sparse signals. Further to the best of our knowledge it is the only algorithm for learning overcomplete dictionaries, that is proven to be (locally) stable for sparsity ranges up to a log factor of the ambient dimension - that is recovery down to a target error $K^{-\ell}$ for sparsity levels $S$ up to $O(\mu^{-2}/(\ell\log K)) = O (d/(\ell\log K)$.\\
As such it improves on related results in terms of computational efficiency, convergence radius and admissible sparsity level, \cite{aganjaneta13}, or in terms of achievable error and admissible sparsity level, \cite{argemamo15}. In the case of noiseless signals, which is the only valid regime for \cite{aganjaneta13}, the sample complexity is in comparison larger by a factor $\eps^{-1}$. However, note that there are information theoretic results indicating that in the case of noisy signals the dependence of the sample complexity on the inverse squared target error $\eps^{-2}$ is optimal, \cite{juelgo14}. For an overview of results for iterative dictionary learning algorithms see Table~1. For a more general overview of theoretical results in dictionary learning see Table~1 in \cite{bagrje14}.

\newcommand*\rot{\rotatebox{60}}
\newcommand{\cmark}{\ding{51}}%
\newcommand{\xmark}{\ding{55}}%
\begin{table}[t]\label{thetable}
\centering
\begin{tabular}{lcccccc}
& \rot{\shortstack[l]{online\\parallelisable}}& \rot{\shortstack[l]{noise\phantom{p}\\stability}}&\rot{\shortstack[l]{convergence\\radius}}
& \rot{\shortstack[l]{admissible\phantom{p}\\sparsity S}} & \rot{\shortstack[l]{sample\\complexity}}
& \rot{\shortstack[l]{achievable\\ error $\epstarget>\cdot$ \phantom{d}}}\\

Agarwal et.al. \cite{aganjaneta13} &\xmark &\xmark & $S^{-2}$& $\min \{\mu^{-1}, d^{1/6}\}$& $ (K^2/S)$ & (0) \\
Arora et.al. \cite{argemamo15}&\cmark &\cmark & $(\log{d})^{-1}$  &$\mu^{-1}$&$SK\,^*$& $\sqrt{S/d}\,^*$  \\
ITKsM&\cmark &\cmark & $ (\log K)^{-1/2}$& $\mu^{-2}$&$K \epstarget^{-2}$&$K^{-\ell} \, (0)$\\
ITKrM&\cmark &\cmark & $S^{-1/2}$& $\mu^{-2}$&$K \epstarget^{-2} (K \epstarget^{-1})$&$K^{-\ell}\, (0)$\\
\\
\end{tabular}
 
\parbox{0.8 \textwidth}{To be read as $O(\cdot)$, non-leading log-factors omitted, noiseless case with $S\leq\mu^{-1}$ in brackets.\\
$^*${\scriptsize valid for Algorithm 2. Algorithm 5 seems to achieve similar errors as ITKsM/ITKrM but at significantly higher computational cost.
Further the dependence of its sample complexity on the target error is not made explicit.}}

\caption{Comparison of theoretical results for iterative dictionary learning algorithms.}
\end{table}

Further we have shown that in synthetic experiments the computationally more involved algorithm ITKrM often converges globally, when initalized with a random dictionary. This together with the fact that the algorithm is also able to learn meaningful dictionaries on image data makes it an attractive low complexity alternative to K-SVD. \\
The global convergence behaviour of ITKrM comes partly as a surprise since for ITKrM we can only prove a convergence radius of the order $O(1/\sqrt{S})$ as opposed to $O(1/\sqrt{\log K)}$ for ITKsM. It also indicates that one might be able to increase the convergence radius of the algorithms by making additional assumptions on the perturbation dictionary, that is the normalised difference between the input and the generating dictionary, such as good conditioning and incoherence like the random perturbations in our experiments. All that then remains to show is that most perturbations have this additional property and that one iteration of ITKrM conserves the property respectively that an additional small corrective step can restore the property. \\
For the theoretical results, another slight disappointment hidden in the O notation is, that both the convergence radius and implicitly also the limiting precision decrease with the dynamic range of the coefficients. This seems unavoidable since the success of thresholding depends on the dynamic range. So while we could improve our results to depend on an average dynamic range instead of the worst dynamic range by assuming a probability distribution on the dynamic range in our proofs, this average dynamic range will remain a limitation. To remove the dependence on the dynamic range we would have to replace thresholding by another sparse approximation method such as (Orthogonal) Matching Pursuit or Basis Pursuit, which is used in \cite{aganjaneta13}. However, the only method that is known to be on average stable for sparsity levels $S\geq \sqrt{d}$ is Basis Pursuit, \cite{tr08}, and it will need some work to extend the corresponding results to perturbed dictionaries, noise and approximation error all at the same time. A maybe less daunting strategy is to extend the stability results for thresholding to iterative (hard) thresholding methods, \cite{blda08,blda09, fo11, jateka14}. Another strategy to overcome large dynamic ranges, we are interested in, which would at the same time remove the requirement of knowing the exact sparsity level, is to extend our results to the case where we can only assure a gap between $c_S$ and $c_{S+T}$ for $T>1$. \\
The most important research directions are concerned with the globality of the results. To get to an efficient algorithm we need to find initialisation strategies, such as in \cite{argemamo15}, that remain cost efficient also for sparsity levels $S=O(\mu^{-2}/(\ell\log K))$. An alternative strategy, we are currently pursuing, is based on the earlier mentioned additional assumptions. If the perturbation dictionary not only has a flat spectrum but is itself incoherent and incoherent to the generating dictionary we expect one step of ITKrM to reduce the perturbation sizes but to keep the perturbation directions roughly the same. Estimating the volume of 'good' perturbations we could then calculate the probability that a random initialisation is successful or, in case this probability is too small, add a corrective step that restores the good properties of the current iterate. \\

\appendices

\section*{Acknowledgements}
This work was supported by the Austrian Science Fund (FWF) under Grant no.~J3335 and Grant no.~Y760
In addition the numerical simulations were supported by the Austrian Ministry of Science (BMWF) as part of the UniInfrastrukturprogramm of the Focal Point Scientific Computing at the University of Innsbruck.
\\ 
Thanks go also to the reviewers for their corrections and helpful suggestions and to the Computer Vision Laboratory of the University of Sassari, Italy, which provided the beautiful surroundings, where the inspirational part of the presented work was done.
\newpage
\section{Proof Sketches}\label{appendix_proofs}

\subsection{Proof of Corollary~\ref{th:itksm_exact}} \label{proof_itksm_exact}
The proof is analogue to the one of Theorem~\ref{th:itksm}. We only need to take into account that without noise we have $C_r=1$ and that in all estimates the constant $B+1$ can be replaced by $B$, since for noise free signals $y_n=\dico x_{c_n, p_n,\sigma_n}$ we have $\|y_n\|_2 \leq B$. Further since the coefficients are strongly $S$-sparse, thresholding using the generating dictionary $\dico$ will always (almost surely) recover the generating support with a margin $u_s \geq (\Delta_S - 2\mu S) c_n(1)$, that is $\min_{k\in I_n} |\ip{\atom_k}{y_n}| \geq \max_{k\notin I_n} |\ip{\atom_k}{y_n}| + u_s$, compare \cite{sc14b}. Therefore the event that thresholding using $\pdico$ fails or that the empirical signs differ from the generating ones is contained in 
\begin{align}\label{eventFsdef}
\mathcal F^s_n&:= \Big\{y_n:  \exists k \mbox{ s.t. } \omega_k\Big|\sum_j  \sigma_n(j) c_n\big(p_n(j)\big) \ip{\atom_j}{ z_k}\Big| \geq \frac{u_s -\frac{\eps^2c_n(S)}{2}}{2} \Big\}
\end{align}
and we get
\begin{align}
\left\| \bar \patom_k - \frac{\gamma_{1,S}}{K}\atom_k\right\|_2 &\leq \frac{2 \sqrt{B}}{N} \sharp \{ n: y_n \in \mathcal F^s_n\} + \left\|\frac{1}{N} \sum_n y_n \, \sigma_n(k) \, \chi(I_n,k) -  \frac{\gamma_{1,S}}{K}\atom_k\right\|_2,
\end{align}
which can be estimated as before.

\subsection{Proof of Theorem~\ref{th:itkrm}} \label{proof_itkrm}
As already mentioned we use the same two step procedure and ideas as in the proof of Theorem~\eqref{th:itksm}. \\
\noindent \emph{Step 1:} We first check how often thresholding with $\pdico$ fails. Assuming thresholding recovers the generating support we show that the difference of the residuals using
$\dico$ or $\pdico$ concentrates around its expectation, which is small. Finally we show that the sum of residuals using $\dico$ converges to a scaled version of $\atom_k$. 
To make the ideas precise we define the thresholding residual based on $\pdico$
\begin{align} \label{defRt}
R^t(\pdico, y_n, k) := \big[y_n - P(\pdico_{I_{\pdico,n}^t}) y_n + P(\patom_k) y_n\big] \cdot \signop(\ip{\patom_k}{y_n}) \cdot  \chi(I_{\pdico,n}^t, k)
\end{align}
and the oracle residual based on the generating support $I_n=p_n^{-1}(\Sset)$, the generating signs $\sigma_n$ and $\pdico$.
\begin{align}\label{defRo}
R^o(\pdico, y_n, k) := \big[y_n - P(\pdico_{I_n}) y_n + P(\patom_k) y_n\big] \cdot \sigma_n(k) \cdot  \chi(I_n,k).
\end{align}
We can now write,
\begin{align}
\bar \patom_k &= \frac{1}{N} \sum_n \left[R^t(\pdico, y_n, k) - R^o(\pdico, y_n, k)\right] +  \frac{1}{N} \sum_n \left[R^o(\pdico, y_n, k) - R^o(\dico, y_n, k)\right]+\frac{1}{N} \sum_n R^o(\dico, y_n, k) \notag \\
&= \frac{1}{N} \sum_n \left[R^t(\pdico, y_n, k) - R^o(\pdico, y_n, k)\right] +  \frac{1}{N} \sum_n \left[R^o(\pdico, y_n, k) - R^o(\dico, y_n, k)\right] \notag\\
&\hspace{2cm}+\frac{1}{N} \sum_n \big[y_n - P(\dico_{I_n}) y_n\big] \cdot \sigma_n(k) \cdot  \chi(I_n,k) + \left( \frac{1}{N} \sum_n \ip{y_n}{\atom_k} \cdot \sigma_n(k) \cdot  \chi(I_n,k)\right) \atom_k .
\end{align}
Abbreviating $s_k=\frac{1}{N} \sum_n \ip{y_n}{\atom_k} \cdot \sigma_n(k) \cdot  \chi(I_n,k)$ we get
\begin{align}
\| \bar \patom_k - s_k \atom_k\|_2 \leq \frac{1}{N} \Big\| \sum_n &\left[R^t(\pdico, y_n, k) - R^o(\pdico, y_n, k)\right] \Big\|_2 \notag\\
&\hspace{2cm}+ \frac{1}{N} \Big\| \sum_n \left[R^o(\pdico, y_n, k) - R^o(\dico, y_n, k)\right] \Big\|_2\notag\\
&\hspace{5cm}+\frac{1}{N} \Big\| \sum_n \big[y_n - P(\dico_{I_n}) y_n\big] \cdot \sigma_n(k) \cdot  \chi(I_n,k) \Big\|_2.\label{itkrmsplit}
\end{align}
We first estimate the norm of the first sum using the fact that the operator $\I_d- P(\pdico_{I_n}) + P(\patom_k)$ is an orthogonal projection and that $\|y_n\|_2\leq\sqrt{B+1}$, 
\begin{align}
\frac{1}{N} \Big\| \sum_n &\left[R^t(\pdico, y_n, k) - R^o(\pdico, y_n, k)\right] \Big\|_2 \leq \frac{2\sqrt{B+1}}{N} \cdot \sharp \{n : R^t(\pdico, y_n, k) \neq R^o(\pdico,y_n, k)\}.
\end{align}
Next note that on the draw of $y_n$ the event that the thresholding residual using $\pdico$ is different from the oracle residual using $\pdico$, $\{y_n : R^t(\pdico, y_n, k) \neq R^o(\pdico,y_n, k) \}$ for any $k$ is again contained in the events $\mathcal E_n \cup \mathcal F_n$ as defined in \eqref{eventEdef}/\eqref{eventFdef},
\begin{align}
\{y_n : R^t(\pdico, y_n, k) \neq R^o(\pdico,y_n, k) \} \subseteq \{y_n : I_{\pdico,n}^t \neq I_n\} \cup \{y_n : \signop(\pdico_{I_n}^\star y_n) \neq \sigma_n(I_n)\} \subseteq \mathcal E_n \cup \mathcal F_n.
\end{align}
Substituting the corresponding bounds into \eqref{itkrmsplit} we get,
\begin{align}
\| \bar \patom_k - s_k \atom_k\|_2 &\leq  \frac{2\sqrt{B+1}}{N} \cdot \sharp \{ n: y_n \in  \mathcal E_n\} +\frac{2\sqrt{B+1}}{N}\cdot \sharp \{ n: y_n \in \mathcal F_n\}\notag\\
&+ \frac{1}{N} \Big\| \sum_n \left[R^o(\pdico, y_n, k) - R^o(\dico, y_n, k)\right] \Big\|_2 +\frac{1}{N} \Big\| \sum_n \big[y_n - P(\dico_{I_n}) y_n\big] \cdot \sigma_n(k) \cdot  \chi(I_n,k) \Big\|_2 .
\end{align}
For the first two terms on the right hand side we use the same estimates as in the proof of Theorem~\ref{th:itksm}. To estimate the remaining two terms on the right hand side as well as $s_k$ we use the corresponding lemmata in the appendix. From Lemma \ref{lemma3b} we know that
\begin{align}
\P\left( \left| \frac{1}{N}\sum_n \chi(I_n,k)  \sigma_n(k) \ip{y_n}{\atom_k}\right| \leq (1-t_0) \frac{C_r \gamma_{1,S}}{K} \right)\leq\exp\left(- \frac{ N t_0^2C_r^2 \gamma_{1,S}^2}{2K(1+ \frac{SB}{K}+S\nsigma^2 + t_0 C_r \gamma_{1,S}\sqrt{B+1}/3)}\right).
\end{align}
From Lemma~\ref{lemma4} we get that if $ S\leq \min \{\frac{K}{98B}, \frac{1}{98\nsigma^2}\}$, $\eps \leq \frac{1}{32\sqrt{S}}$ and $\eps_\delta \leq \frac{1}{24(B+1)}$ then
\begin{align}
&\P\left(\frac{1}{N} \left\| \sum_n \left[R^o(\pdico, y_n, k)-R^o(\dico, y_n, k) \right]\right\|_2 \geq \frac{ C_r \gamma_{1,S}}{K}(0.381\eps + t_3)\right)\notag\\
&\hspace{4cm}\leq \exp\left(- \frac{ t_3 C^2_r \gamma^2_{1,S} N}{40K\max\{S,B+1\}} \min\left\{\frac{t_3}{\eps^2 +  \eps_\delta \left(1-\gamma_{2,S}+ d \nsigma^2\right)/160}, \frac{5}{3}\right\} +\frac{1}{4}\right).
\end{align} 
Finally from Lemma~\ref{lemma3a} we know that for $0\leq t_4 \leq 1-\gamma_{2,S} + d\nsigma^2$, we have
\begin{align}
&\P\left( \left\| \frac{1}{N}\sum_n \big[y_n - P(\dico_{I_n})y_n \big]\cdot \sigma_n(k) \cdot \chi(I_n,k) \right\|_2 \geq   \frac{C_r\gamma_{1,S}}{ K}\, t_4 \right)\notag\\
&\hspace{5cm}\leq \exp\left(- \frac{t_4^2C^2_r\gamma_{1,S}^2 N}{8 K \max\{S,B+1\}\left(1-\gamma_{2,S} + d\nsigma^2\right)}+\frac{1}{4}\right).
\end{align}
Thus with high probability we have
\begin{align}
\left\| \bar \patom_k - s_k \atom_k\right\|_2 &\leq \frac{C_r\gamma_{1,S}}{K} \left(\eps_{\mu,\nsigma} + t_1 + \tau \eps + t_2 + 0.381 \eps + t_3 + t_4\right) \quad \mbox{and} \quad s_k \geq (1-t_0) \frac{C_r \gamma_{1,S}}{K}.
\end{align}
To be more precise, if we choose a target precision $\epstarget \geq 8\eps_{\mu,\nsigma}$ and set $t_1=\epstarget/24$, $t_2=t_3= \max\{\epstarget, \eps\}/24$, $\tau=1/24$, $t_4=\epstarget/8$ and $t_0=1/50$ we get 
\begin{align}
\max_k \left\| \bar \patom_k - \frac{C_r\gamma_{1,S}}{K}\atom_k\right\|_2 &\leq 0.8\cdot \frac{C_r\gamma_{1,S}}{K} \max\{\epstarget,\eps\}  \quad \mbox{and} \quad \min_k s_k \geq 0.98 \cdot \frac{C_r \gamma_{1,S}}{K}.
\end{align}
except with probability
\begin{align}
\exp\left( \frac{-C_r\gamma_{1,S} N\epstarget}{336\, K\sqrt{B+1} }\right) + \exp\left( \frac{-C_r\gamma_{1,S} N \max\{\epstarget, \eps\} }{144\,K\sqrt{B+1} }\right)+K\exp\left(\frac{-C^2_r\gamma_{1,S}^2N}{K(5103 +34\,C_r\gamma_{1,S} \sqrt{B+1})}\right) \notag\\
+ 2K\exp\left(\frac{-C^2_r\gamma_{1,S}^2N\epstarget^2}{512K \max\{S,B+1\} \left(1-\gamma_{2,S} + d\nsigma^2\right)}\right) + 
2K\exp\left(\frac{-C^2_r\gamma_{1,S}^2N\max\{\epstarget, \eps\}^2}{576K\max\{S,B+1\}\left(\eps +1 - \gamma_{2,S} + d\nsigma^2\right)}\right). \notag
\end{align}
Note that in case the target precision $\epstarget$ is larger than $\eps_\delta$, as happens for instance as soon as $\beta_S\leq \frac{1}{7\sqrt{S}}$ and therefore $\eps_{\mu, \nsigma}\geq \eps_\delta $, the last term in the sum above reduces to
\begin{align}
2K\exp\left(\frac{-C^2_r\gamma_{1,S}^2N\max\{\epstarget, \eps\}}{576 K \max\{S,B+1\}\left(2-\gamma_{2,S} + d\nsigma^2\right)}\right). 
\end{align}
Lemma~\ref{lemma_rescale} then again implies that
\begin{align}
d(\bar \pdico, \dico)=\max_k \left\| \frac{\bar \patom_k}{\|\bar \patom_k\|_2} - \atom_k\right\|_2 \leq 0.92 \max\{\epstarget,\eps\}.
\end{align}
\noindent \emph{Step 2:} The second step is analogue to the one in the proof of Theorem~\ref{th:itksm}.

\subsection{Proof of Theorem~\ref{th:itkrm_exact}}\label{proof_itkrm_exact}
We follow the proof of Theorem~\ref{th:itkrm} but take into account that
in case of exactly $S$-sparse, noiseless signals the bound \eqref{itkrmsplit} reduces to
\begin{align}
\| \bar \patom_k - s_k \atom_k\|_2 &\leq  \frac{2\sqrt{B}}{N}\cdot \sharp \{ n: y_n \in \mathcal F^s_n\}+ \frac{1}{N} \Big\| \sum_n \left[R^o(\pdico, y_n, k) - R^o(\dico, y_n, k)\right] \Big\|_2.
\end{align}
Since the relative gap $\Delta_S > 2 \mu S$ we get $\delta_S \leq \mu S \leq \frac{1}{2}$ and by Lemma~\ref{lemma1b} 
\begin{align}
&\P\left(\sharp \{ n: y_n \in  \mathcal F^s_n\} \geq  \frac{\gamma_{1,S}N}{ 2K\sqrt{B}}\cdot(\tau \eps + t_2) \right)\leq \exp\left( \frac{-t_2^2\gamma_{1,S} N}{2K\sqrt{B}\,(2\tau \eps + t_2) }\right),
\end{align}
whenever 
\begin{align}
\eps \leq \frac{\Delta_S - 2\mu S}{\sqrt{12}\left(\frac{1}{4} +\sqrt{\log\left(\frac{23K^2 \sqrt{B}}{(\Delta_S - 2\mu S) \gamma_{1,S}\tau}\right)}\right)}.
\end{align}
Further by Lemma~\ref{lemma4}
\begin{align}
&\P\left(\frac{1}{N} \left\| \sum_n \left[R^o(\pdico, y_n, k)-R^o(\dico, y_n, k) \right]\right\|_2 \geq \frac{ \gamma_{1,S}}{K}(1 \eps + t_3)\right) \leq \exp\left(- \frac{ t_3 \gamma^2_{1,S} N}{32\eps K\max\{S,B\}} \min\left\{\frac{t_3}{\eps}, 1\right\} +\frac{1}{4}\right),\notag
\end{align}
and again by \ref{lemma3b}
\begin{align}
\P\left( \left| \frac{1}{N}\sum_n \chi(I_n,k)  \sigma_n(k) \ip{y_n}{\atom_k}\right| \leq (1-t_0) \frac{C_r \gamma_{1,S}}{K} \right)\leq\exp\left(- \frac{ N t_0^2 \gamma_{1,S}^2}{2K(1+ \mu^2(S-1) + t_0 \gamma_{1,S}\sqrt{B}/3)}\right).
\end{align}
Thus with high probability we have
\begin{align}
\left\| \bar \patom_k - s_k \atom_k\right\|_2 &\leq \frac{\gamma_{1,S}}{K} \left(\tau \eps + t_2 + 0.611 \eps + t_3 \right) \quad \mbox{and} \quad s_k \geq (1-t_0) \frac{\gamma_{1,S}}{K}.
\end{align}
The final result follows as before from setting $t_0=1/50$, $\tau = 1/24$, $t_2=\max\{\tilde \eps,\eps\}/24$ and $t_3=2t_2$.


\section{Probability Estimates \& Technicalities}\label{appendix}

\begin{theorem}[Vector Bernstein, \cite{kugr14, gr11, leta91}]\label{vectorbernstein}
Let $(v_n)_n \in \R^d$ be a finite sequence of independent random vectors. If 
$\|v_n\|_2 \leq M$ almost surely, $\|\E(v_n)\|_2\leq m_1$ and $\sum_n \E(\|v_n\|_2^2)\leq m_2$, then for all $0\leq t \leq m_2/(M+m_1)$
\begin{align}
\P\left(\left\| \sum_n v_n - \sum_n \E(v_n) \right\|_2 \geq t\right)\leq \exp\left(- \frac{t^2}{8m_2}+\frac{1}{4}\right),
\end{align}
and in general
\begin{align}
\P\left(\left\| \sum_n v_n - \sum_n \E(v_n) \right\|_2 \geq t\right)\leq \exp\left(- \frac{t}{8}\cdot \min\left\{\frac{t}{m_2}, \frac{1}{M+m_1}\right\}+\frac{1}{4}\right).
\end{align}
\end{theorem}
Note that the general statement is simply a consequence of the first part, since for $t\geq m_2/(M+m_1)$ we can choose $m_2=t (M+m_1)$. \\
For the simple case of random variables we also state a scalar version of Bernstein's inequality leading to better constants. 
\begin{theorem}[Scalar Bernstein, \cite{bennett62}] \label{scalarbernstein}
Let $v_n\in \R$, $n=1\ldots N$ be a finite sequence of independent random variables. If 
$\E(v^2_n)\leq m $ and $ \E(|v_n|^k )\leq \frac{1}{2} k!\,m M^{k-2}$ for all $k > 2$ then for all $t>0$
\begin{align}
\P\left(\left| \sum_n v_n - \sum_n \E(v_n) \right| \geq t\right)\leq \exp\left(- \frac{t^2}{2(N m + Mt) } \right).\notag
\end{align}
\end{theorem}

\begin{lemma}\label{lemma1a}
 For $y_n$ following model~\eqref{noisymodel2} with coefficients that have an absolute gap $\beta_S$ we have,
\begin{align}
&\P\left(\sharp \{ n: y_n \in  \mathcal E_n\} \geq  \frac{C_r\gamma_{1,S}N}{ 2K\sqrt{B+1}}\cdot(\eps_{\mu,\nsigma} + t) \right)\leq \exp\left( \frac{-t^2C_r\gamma_{1,S} N}{2K\sqrt{B+1}\,(2\eps_{\mu, \nsigma} + t) }\right),
\end{align}
where $\eps_{\mu,\nsigma} =\frac{8K^2\sqrt{B+1}}{C_r \gamma_{1,S}} \exp\left(\frac{-\beta_S^2}{98\max\{ \mu^2,\nsigma^2\}}\right)$.
\end{lemma}
\begin{proof}
We apply Theorem~\ref{scalarbernstein} to the sum of indicator functions $\bf{1}_{\mathcal E_n}$ to get
\begin{align}
\P\left(\sharp \{ n: y_n \in  \mathcal E_n\} \geq \sum_n \P(\mathcal E_n) + tN\right) \leq \exp \left(\frac{ -t^2N^2}{2 \sum_n \P(\mathcal E_n) + tN}\right).
\end{align}
To estimate $\P(\mathcal E_n)$ we apply Hoeffding's inequality to \eqref{eventEdef} resp. use the subgaussian property of $r_n$. Omitting subscripts for simplicity and abbreviating $u=c(S)-c(S+1)$ we get,
\begin{align}
\P(\mathcal E) &\leq \sum_k \P\left( \Big| \sum_{j \neq k} \sigma(j) c\big(p(j)\big) \ip{\atom_j}{\atom_k}\Big| \geq \frac{u}{7} \right) + \sum_k \P\left(|\ip{r}{\atom_k}| \geq  \frac{u}{7}\right)\notag\\
&\leq \sum_k 2\exp\left(\frac{u^2}{98\sum_{j \neq k} c\big(p(j)\big)^2 |\ip{\atom_j}{\atom_k}|^2}\right)+2K \exp\left(\frac{-u^2}{98 \nsigma^2}\right)\notag\\
&\leq 2K \exp\left(\frac{-\beta_S^2}{98 \mu^2}\right)+2K \exp\left(\frac{-\beta_S^2}{98 \nsigma^2}\right)\notag\\
& \leq 4K\exp\left(\frac{-\beta_S^2}{98\max\{ \mu^2,\nsigma^2\}}\right) 
 =  \frac{C_r\gamma_{1,S}}{ 2K\sqrt{B+1}} \cdot \eps_{\mu,\nsigma}.
\end{align}
The result follows from the substitution $t\rightarrow  \frac{C_r\gamma_{1,S}}{ 2K\sqrt{B+1}}\, t$. \\
\end{proof}

\begin{lemma}\label{lemma1b}
(a) For $y_n$ following model~\eqref{noisymodel2} with coefficients that have a relative gap $\Delta_S$ we have,
\begin{align}
&\P\left(\sharp \{ n: y_n \in  \mathcal F_n\} \geq  \frac{C_r\gamma_{1,S}N}{ 2K\sqrt{B+1}}\cdot(\tau \eps + t) \right)\leq \exp\left( \frac{-t^2C_r\gamma_{1,S} N}{2K\sqrt{B+1}\,(2\tau \eps + t) }\right),
\end{align}
whenever 
\begin{align}
\eps \leq \frac{\Delta_S}{\sqrt{98B}\left(\frac{1}{4} +\sqrt{\log\left(\frac{106K^2(B+1)}{\Delta_S C_r \gamma_{1,S}\tau}\right)}\right)}.\label{epsmaxapp}
\end{align}
(b) For $y_n$ following model~\eqref{noisymodel2} with coefficients that have a relative gap $\Delta_S\geq 2\mu S$ we have,
\begin{align}
&\P\left(\sharp \{ n: y_n \in  \mathcal F^s_n\} \geq  \frac{\gamma_{1,S}N}{ 2K\sqrt{B}}\cdot(\tau \eps + t) \right)\leq \exp\left( \frac{-t^2\gamma_{1,S} N}{2K\sqrt{B}\,(2\tau \eps + t) }\right),
\end{align}
whenever 
\begin{align}
\eps \leq \frac{\Delta_S - 2\mu S}{\sqrt{8B}\left(\frac{1}{4} +\sqrt{\log\left(\frac{19K^2 B}{(\Delta_S - 2\mu S) \gamma_{1,S}\tau}\right)}\right)}.\label{epsmaxapp2}
\end{align}
\end{lemma}
\begin{proof}
We apply Theorem~\ref{scalarbernstein} to the sum of indicator functions $\bf{1}_{\mathcal F^{(s)}_n}$ to get
\begin{align}
\P\left(\sharp \{ n: y_n \in  \mathcal F^{(s)}_n\} \geq \sum_n \P(\mathcal F^{(s)}_n) + tN\right) \leq \exp \left(\frac{ -t^2N^2}{2 \sum_n \P(\mathcal F^{(s)}_n) + tN}\right).
\end{align}
To estimate $\P(\mathcal F^{(s)}_n)$ we again apply Hoeffding's inequality this time to \eqref{eventFdef}/\eqref{eventFsdef} resp. use the subgaussian property of $r_n$. Omitting subscripts and using the short hand $u=c(S)-c(S+1)$ and $u_s =(\Delta_S - 2\mu S) c(1)$ we get,
\begin{align}
\P(\mathcal F) &\leq \sum_k \P\left(\omega_k \Big| \sum_{j \neq k} \sigma(j) c\big(p(j)\big) \ip{\atom_j}{z_k}\Big| \geq \frac{u}{7} -\frac{\eps^2c(S)}{6}\right) + \sum_k \P\left(\omega_k |\ip{r}{z_k}| \geq  \frac{u}{14}-\frac{\eps^2c(S)}{12}\right)\notag\\
&\leq \sum_k 2\exp\left(\frac{-\left(u-\frac{7\eps^2c(S)}{6}\right)^2}{98\omega_k^2\sum_{j \neq k} c\big(p(j)\big)^2 |\ip{\atom_j}{z_k}|^2}\right)+2K \exp\left(\frac{-\left(u-\frac{7\eps^2c(S)}{6}\right)^2}{4\cdot 98 \nsigma^2}\right)\notag\\
&\leq 2K \exp\left(\frac{-\left(u-\frac{7\eps^2c(S)}{6}\right)^2}{98 \eps^2\min\{ c(1)^2 B,1\}}\right)+2K\exp\left(\frac{-\left(u-\frac{7\eps^2c(S)}{6}\right)^2}{4\cdot98 \eps^2\nsigma^2}\right)\notag\\
&\leq 5K\exp\left(\frac{-(c(S)-c(S+1))^2}{98\eps^2 c(1)^2B}\right) \leq 5K\exp\left(\frac{-\Delta_S^2}{98\eps^2 B}\right).
\end{align}
From Lemma A.3 in \cite{sc14} we further know that condition \eqref{epsmaxapp} implies 
\begin{align}
5K\exp\left(\frac{-\Delta_S^2}{98\eps^2 B}\right) \leq \frac{C_r\gamma_{1,S}}{ 2K\sqrt{B+1}}\cdot \tau \eps,
\end{align}
and the result in (a) follows again from the substitution $t\rightarrow  \frac{C_r\gamma_{1,S}}{ 2K\sqrt{B+1}}\, t$.\\
Similarly we get 
\begin{align}
\P(\mathcal F^s) &\leq \sum_k \P\left(\omega_k \Big| \sum_{j \neq k} \sigma(j) c\big(p(j)\big) \ip{\atom_j}{z_k}\Big| \geq \frac{u_s}{2} -\frac{\eps^2c(S)}{4}\right) \notag \\
&\leq
2K \exp\left(\frac{-\left((\Delta_S - 2\mu S) c(1)-\frac{\eps^2c(S)}{2}\right)^2}{8 \eps^2\min\{ c(1)^2 B,1\}}\right)\leq 3K\exp\left(\frac{-(\Delta_S - 2\mu S)^2}{8\eps^2 B}\right) \leq \frac{\gamma_{1,S}}{ 2K\sqrt{B}}\cdot \tau \eps,
\end{align}
whenever \eqref{epsmaxapp2} holds and the result in (b) follows from the substitution $t\rightarrow  \frac{\gamma_{1,S}}{ 2K\sqrt{B}}\, t$.\\
Finally note that another (messier) way to bound $\sum_{j \neq k} c\big(p(j)\big)^2 |\ip{\atom_j}{z_k}|^2$ is
\begin{align}
\sum_{j \neq k} c\big(p(j)\big)^2 |\ip{\atom_j}{z_k}|^2\leq \min\{ c(1)^2 \|\dico_I\|^2_{2,2} + 1-\gamma_{2,S}, c(1)^2 \|\dico_I\|^2_{2,2} + c(S+1)^2 B\}.
\end{align}
However, in the case of exactly S-sparse signals these can lead to better (and again clean) estimates, such as $c(1)^2 (1+\mu S)$ or  $c(1)^2 (1+\delta_S)$ if $\dico$ has isometry constant $\delta_S<1$.  
\end{proof}

\begin{remark}
The last two lemmata are used to prove that, once the perturbed dictionary $\pdico$ is within radius $O(1/\log(K))$ of the generating dictionary $\dico$, thresholding will always succeed in recovering the full generating support, even for $S=O(\mu^{-2})$. Without assuming random signs, we can still get that thresholding recovers the generating support once $\pdico$ is within radius $O(1/\sqrt{S})$ for reduced sparsity levels $S=O(\mu^{-1})$. 
\end{remark}

\begin{lemma}\label{lemma2}
For $y_n=\frac{ \dico x_{c_n, p_n,\sigma_n} +\noise_n}{\sqrt{1+\|\noise_n \|_2^2}}$ as in model~\eqref{noisymodel2} and $0\leq t \leq \frac{\sqrt{S}}{\sqrt{B}+2}$ we have
\begin{align}
\P\left( \left\| \frac{1}{N}\sum_n \frac{ \dico x_{c_n, p_n,\sigma_n} +\noise_n}{\sqrt{1+\|\noise_n \|_2^2}}\cdot \sigma_n(k) \cdot \chi(I_n,k) - \frac{C_r \gamma_{1,S}}{K} \atom_k \right\|_2 \geq \frac{C_r\gamma_{1,S}}{ K}\, t \right)\leq \exp\left(- \frac{t^2C^2_r\gamma_{1,S}^2 N}{8SK}+\frac{1}{4}\right).
\end{align}
\end{lemma}
\begin{proof}
We apply Theorem~\ref{vectorbernstein} to $v_n=\frac{ \dico x_{c_n, p_n,\sigma_n} +\noise_n}{\sqrt{1+\|\noise_n \|_2^2}}\cdot \sigma_n(k) \cdot \chi(I_n,k)$. Since the $v_n$ are identically distributed we drop the index $n$ for our estimates.
Remembering that $I=p^{-1}(\Sset)$ we get,
\begin{align}\label{eq:randsign1}
\E(v) &= \E_{c,p,\sigma, r}\left(\frac{\chi(I,k)}{\sqrt{1+\|\noise\|_2^2}} \left(\sum_j \atom_j c\big(p(j)\big) \sigma(j) \cdot \sigma(k) + r \cdot \sigma(k)\right) \right)\notag\\
&=\E_{c,p,r}\left(\frac{\chi(\Sset,p(k))\cdot c\big(p(k)\big)}{\sqrt{1+\|\noise \|_2^2}} \: \atom_k  \right)\notag\\
&=\E_{r}\left(\frac{1}{\sqrt{1+\|\noise \|_2^2}}\right) \E_c\left(\frac{c(1) + \ldots +c(S)}{K}\right)  \atom_k 
 = \frac{C_r \gamma_{1,S}}{K}\: \atom_k,
\end{align}
and $\|\E(v)\|_2\leq \sqrt{S}/K$. Together with the estimates,
\begin{align*}
&\E\left(\|v\|_2^2\right)=\E\left( \frac{\chi(I,k)}{1+\|\noise\|_2^2} \cdot \left(\|  \dico x_{c, p,\sigma}\|_2^2 +
\ip{\dico x_{c, p,\sigma}}{r} +\|r\|_2^2 \right)\right)=\E\left( \chi(I,k)\right)=\frac{S}{K}\\
&\mbox{ and }\qquad
\|v\|_2\leq \frac{\| \dico x_{c, p,\sigma} +\noise\|_2}{\sqrt{1+\|\noise \|_2^2}} \leq \frac{\sqrt{B}+\|r\|_2}{\sqrt{1+\|\noise \|_2^2}}\leq \sqrt{B+1}, 
\end{align*}
this leads to 
\begin{align}
\P\left( \left\| \frac{1}{N}\sum_n \frac{ \dico x_{c_n, p_n,\sigma_n} +\noise_n}{\sqrt{1+\|\noise_n \|_2^2}}\cdot \sigma_n(k) \cdot \chi(I_n,k) - \frac{C_r \gamma_{1,S}}{K} \atom_k \right\|_2 \geq t \right)\leq \exp\left(- \frac{t^2K N}{8 S}+\frac{1}{4}\right),
\end{align}
for $0\leq t \leq \frac{S}{K(\sqrt{B+1}+\frac{S}{K})}$. The final statements follows from the substitution $t\rightarrow  \frac{C_r\gamma_{1,S}}{ K}\, t$ and simplifications.\\
\end{proof}

\begin{remark}\label{rem:lemma2}
Note that for Eq.~\eqref{eq:randsign1} in the above proof, we have used the sign invariance in our model but not the permutation invariance. For very small sparsity levels we can also get a stable version of the lemma using only the permutation invariance. Assume for simplicity that $\dico$ is an orthonormal basis and that the sparse coefficients are constant, $c_k\equiv c$ for $k\leq S$ and zero else.
In this worst case scenario where the signs never cancel out we get 
\begin{align}
\E(v)= c \left( \atom_k + \frac{S-1}{d-1} \sum_{j \neq k} \atom_j \right) \qquad \mbox{and} \qquad \| \E(v)\|_2= c \sqrt{1+ \frac{(S-1)^2}{d-1}},
\end{align}
which implies that the atoms can be learned up to a precision $O(S^2/d)$. A relaxed condition replacing sign and permutation invariance could be that the coefficient sequences $x$ satisfy $\mathbb{E}\left( x(j) \operatorname{sign}({x(k)})| k \in I\right) \ll \mathbb{E}( |x(k)| \, | k \in I)$ for $I$ containing the indices of the $S$ largest coordinates in absolute value, that is $\min_{i\in I} |x(i)| > \max_{j\notin I} |x(j)|$. This condition is quite natural as it basically prevents two atoms $\atom_k$ and $\atom_j$ from always appearing together in the same ratio $x(k):x(j) = a:b$. In this case they could simply be replaced by two copies of the same atom, $\tilde\atom_j=\tilde\atom_k= a \atom_k + b\atom_j$ which would increase the response criterion on which ITKsM is based, see \cite{sc14b}.
\end{remark}

\begin{lemma}\label{lemma3b}
For $y_n=\frac{ \dico x_{c_n, p_n,\sigma_n} +\noise_n}{\sqrt{1+\|\noise_n \|_2^2}}$ as in model~\eqref{noisymodel2} we have
\begin{align}
\P\left( \left| \frac{1}{N}\sum_n \chi(I_n,k)  \sigma_n(k) \ip{y_n}{\atom_k}\right| \leq (1-t) \frac{C_r \gamma_{1,S}}{K} \right)\leq\exp\left(- \frac{ N t^2C_r^2 \gamma_{1,S}^2}{2K(1+ \frac{SB}{K}+S\nsigma^2 + t C_r \gamma_{1,S}\sqrt{B+1}/3)}\right).
\end{align}
\end{lemma}
\begin{proof}
We apply Theorem~\ref{scalarbernstein} to $v_n=\chi(I_n,k)  \sigma_n(k) \ip{y_n}{\atom_k}$, as usual dropping the index $n$ in the estimates for conciseness. For the expectation we get 
\begin{align}
\E(v) &= \E_{c,p,\sigma, r}\left(\frac{\chi(I,k)}{\sqrt{1+\|\noise\|_2^2}} \:  \left(\sum_j c\big(p(j)\big) \sigma(j) \ip{\atom_j}{\atom_k}\cdot \sigma(k) + \ip{r}{\atom_k} \cdot \sigma(k)\right) \right)\notag\\
&=\E_{c,p,r}\left(\frac{\chi(\Sset,p(k))\cdot c\big(p(k)\big)}{\sqrt{1+\|\noise \|_2^2}} \right)
 = \frac{C_r \gamma_{1,S}}{K}.
\end{align}
We further estimate the second moment $m$ as
\begin{align}
\E\left(v^2\right)&=\E_{c,p,\sigma, r}\left( \frac{\chi(I,k)}{1+\|\noise\|_2^2}  \Big(\sum_j c\big(p(j)\big) \sigma(j) \ip{\atom_j}{\atom_k}+ \ip{r}{\atom_k}\Big)^2 \right)\notag\\
&\leq\E_{c,p}\left(\chi(I,k) \cdot \left(\sum_j c\big(p(j)\big)^2 |\ip{\atom_j}{\atom_k}|^2 +\E_r\left(|\ip{r}{\atom_k}|^2\right) \right)\right) \notag\\
&\leq\E_{c,p}\left(\chi(I,k) \cdot \left(\frac{\gamma_{2,S}}{S} + \frac{1-\frac{\gamma_{2,S}}{S}}{K-1}\sum_{j\in I , j\neq k} |\ip{\atom_j}{\atom_k}|^2 + \nsigma^2 \right)\right) \leq \frac{S}{K} \cdot \left( \frac{\gamma_{2,S}}{S} +\frac{B}{K} + \nsigma^2 \right). \end{align}
In the case of exactly $S$-sparse signals, where $\gamma_{2,S}=1$ we get the alternative bound, $\E\left(v^2\right)\leq \frac{1}{K} (1 + (S-1)\mu^2 +S\nsigma^2)$.
Since $|v| \leq  |\ip{y}{\atom_k}|\leq \|y\|_2 \leq \sqrt{B+1}$ we can choose $M= \frac{\sqrt{B+1}}{3}$.
\end{proof}

\begin{lemma}\label{lemma3a}
For $y_n=\frac{ \dico x_{c_n, p_n,\sigma_n} +\noise_n}{\sqrt{1+\|\noise_n \|_2^2}}$ as in model~\eqref{noisymodel2}
\begin{align}
&\P\left( \left\| \frac{1}{N}\sum_n \left(y_n - P(\dico_{I_n})y_n \right)\cdot \sigma_n(k) \cdot \chi(I_n,k) \right\|_2 \geq   \frac{C_r\gamma_{1,S}}{ K}\, t \right)\notag \\
&\hspace{4cm}\leq \exp\left(- \frac{tC^2_r\gamma_{1,S}^2 N}{8K\max\{S,B+1\}}\max\left\{ \frac{t}{1-\gamma_{2,S} + d\nsigma^2},1\right\}+\frac{1}{4}\right).
\end{align}
\end{lemma}
\begin{proof}
We apply Theorem~\ref{vectorbernstein} to $v_n=\left(y_n - P(\dico_{I_n})y_n \right)\cdot \sigma_n(k) \cdot \chi(I_n,k)$. For brevity we again drop the index $n$ in the estimates and define the orthogonal projection $Q(\dico_{I})=\I_d - P(\dico_{I})$.
For the expectation we get 
\begin{align}\label{eqsignstable1}
\E(v) &= \E_{c,p,\sigma, r}\left(\frac{\chi(I,k)}{\sqrt{1+\|\noise\|_2^2}} \: Q(\dico_{I}) \left(\sum_j \atom_j c\big(p(j)\big) \sigma(j) \cdot \sigma(k) + r \cdot \sigma(k)\right) \right)\notag\\
&= \E_{c,p, r}\left(\frac{\chi(I,k)}{\sqrt{1+\|\noise\|_2^2}} \:c\big(p(k)\big) Q(\dico_{I})\atom_k\right) =0,
\end{align}
and for the second moment
\begin{align}
\E\left(\|v\|_2^2\right)&=\E_{c,p,\sigma, r}\left( \frac{\chi(I,k)}{1+\|\noise\|_2^2} \cdot \left(\| Q(\dico_{I}) \dico x_{c, p,\sigma}\|_2^2 +
\ip{Q(\dico_{I}) \dico x_{c, p,\sigma}}{Q(\dico_{I}) r} +\|Q(\dico_{I}) r\|_2^2 \right)\right) \notag\\
&\leq\E_{c,p}\left(\chi(I,k) \cdot \left(\sum_j c\big(p(j)\big)^2 \| Q(\dico_{I})\atom_j\|_2^2 +\E_r\left(\frac{\|Q(\dico_{I}) r\|_2^2}{1+\|\noise\|_2^2}\right) \right)\right) \notag\\
&\leq\E_{c,p}\left(\chi(I,k) \cdot \left(\sum_{j\notin I} c\big(p(j)\big)^2 +\min\{1,(d-S)\nsigma^2\} \right)\right) \leq \frac{S}{K} \cdot \left( 1-\gamma_{2,S} +d\nsigma^2  \right) .
\end{align}
Since $v$ is bounded, 
\begin{align}
\|v\|_2\leq \frac{\| Q(\dico_{I})(\dico x_{c, p,\sigma} +\noise)\|_2}{\sqrt{1+\|\noise \|_2^2}} \leq \frac{\sqrt{B(1-\gamma_{2,S,\min} )}+\|\noise\|_2}{\sqrt{1+\|\noise \|_2^2}}\leq \sqrt{B(1-\gamma_{2,S,\min} )+1} \leq \sqrt{B+1},
\end{align}
we get for $t\rightarrow \frac{C_r\gamma_{1,S}}{ K}\, t $
\begin{align}
&\P\left( \left\| \frac{1}{N}\sum_n \left(y_n - P(\dico_{I_n})y_n \right)\cdot \sigma_n(k) \cdot \chi(I_n,k) \right\|_2 \geq   \frac{C_r\gamma_{1,S}}{ K}\, t \right) \notag\\
&\hspace{4cm}\leq \exp\left(- \frac{tC_r\gamma_{1,S} N}{8K} \max\left\{ \frac{tC_r\gamma_{1,S}}{S(1-\gamma_{2,S} + d\nsigma^2)},\frac{1}{\sqrt{B+1}} \right\}+\frac{1}{4}\right)\notag\\
&\hspace{4cm}\leq \exp\left(- \frac{tC^2_r\gamma^2_{1,S} N}{8K} \max\left\{ \frac{t}{S(1-\gamma_{2,S} + d\nsigma^2)},\frac{1}{C_r \gamma_{1,S} \sqrt{B+1}} \right\}+\frac{1}{4}\right).
\end{align}
The final bound follows from the fact that $C_r<1$ and $\gamma_{1,S}\leq \sqrt{S}$.\\
\end{proof}

\begin{remark}
Note that for the abov lemma neither the sign nor the permutation invariance are crucial. Without both assumptions we could still get a stable version of the lemma because we can bound $\E(v)$ by the residual energy $\|y_n - P(\dico_{I_n} y_n\|_2$, which should be small if the signals are assumed to be $S$-sparse. To get perfect recoverability $\E(v)$ we could make the natural assumption that in expectation the residuals $a_n=y_n - P(\dico_{I_n}) y_n=Q(\dico_{I_n}) \dico x_n$ are uncorrelated with the sign of the k-th coefficient $x_n(k)$ whenever $k\in I_n$ , $\E\left(a_n \signop(x(k)) \chi(I_n,k) \right)=0$. Indeed if this is not the case it means that the signals can be even better sparsely approximated if the atom $\atom_k$ is distorted towards this signed residual mean.
\end{remark}

\begin{lemma}\label{lemma4}
Assume that $y_n=\frac{ \dico x_{c_n, p_n,\sigma_n} +\noise_n}{\sqrt{1+\|\noise_n \|_2^2}}$ follows the random model in~\eqref{noisymodel2}. Assume $ S\leq \min \{\frac{K}{98B}, \frac{1}{98\nsigma^2}\}$ and $d(\dico,\pdico)=\eps \leq \frac{1}{32\sqrt{S}}$.\\
(a) If $\eps_\delta:=K  \exp\left(-\frac{1}{4741\mu^2 S} \right)\leq  \frac{1}{48(B+1)}$ we have
\begin{align}
&\P\left(\frac{1}{N} \left\| \sum_n \left[R^o(\pdico, y_n, k)-R^o(\dico, y_n, k) \right]\right\|_2 \geq \frac{ C_r \gamma_{1,S}}{K}(0.381\eps + t)\right)\notag\\
&\hspace{4cm}\leq \exp\left(- \frac{ t C_r \gamma_{1,S} N}{8K} \min\left\{\frac{t C_r \gamma_{1,S}}{S\left[5\eps^2 +  \eps_\delta \left(1-\gamma_{2,S}+ d \nsigma^2\right)/32\right]}, \frac{1}{3\sqrt{B+1}}\right\} +\frac{1}{4}\right).
\end{align} 
(b) If $\gamma_{2,S}=1, \nsigma=0$ together with $\eps_\delta\leq  \frac{1}{48(B+1)}$ or $\delta_S(\dico) \leq 1/4$ this reduces to 
\begin{align}
&\P\left(\frac{1}{N} \left\| \sum_n \left[R^o(\pdico, y_n, k)-R^o(\dico, y_n, k) \right]\right\|_2 \geq \frac{ C_r \gamma_{1,S}}{K}(0.381 \eps + t)\right)\notag\\
&\hspace{4cm}\leq \exp\left(- \frac{ t \gamma^2_{1,S} N}{32\eps K\max\{S,B\}} \min\left\{\frac{t}{\eps}, 1\right\} +\frac{1}{4}\right).
\end{align}
(c) If $\gamma_{2,S}=1, \nsigma=0$ and $\delta_S(\dico) \leq 1/2$ we have
\begin{align}
&\P\left(\frac{1}{N} \left\| \sum_n \left[R^o(\pdico, y_n, k)-R^o(\dico, y_n, k) \right]\right\|_2 \geq \frac{ \gamma_{1,S}}{K}(0.611 \eps + t)\right)\notag\\
&\hspace{4cm}\leq \exp\left(- \frac{ t \gamma^2_{1,S} N}{32\eps K\max\{S,B\}} \min\left\{\frac{t}{\eps}, 1\right\} +\frac{1}{4}\right).
\end{align}
\end{lemma}
\begin{proof}
We apply Theorem~\ref{vectorbernstein} to $v_n=R^o(\pdico, y_n, k)-R^o(\dico, y_n, k)$.  Again we drop the index $n$ in the estimates.
Remembering the definition of $R^o(\pdico, y_n, k)$ in \eqref{defRo} we first expand $v$ as
\begin{align}
v&=\big(y_n - P(\pdico_{I_n}) y_n + P(\patom_k) y_n\big) \cdot \sigma_n(k) \cdot  \chi(I_n,k) - \big(y_n - P(\dico_{I_n}) y_n + P(\atom_k) y_n\big) \cdot \sigma_n(k) \cdot  \chi(I_n,k)\notag \\
&=\left[P(\dico_{I})-P(\pdico_{I})- P(\atom_k) + P(\patom_k) \right]y \cdot \sigma(k) \cdot  \chi(I,k).
\end{align}
Abbreviate $T(I,k):=P(\dico_{I})-P(\pdico_{I})- P(\atom_k) + P(\patom_k)$. Taking the expectation we get
\begin{align}\label{signstable2}
\E(v) &= \E_{c,p,\sigma, r}\left(\frac{\chi(I,k)}{\sqrt{1+\|\noise\|_2^2}} \: T(I,k) \left(\sum_j \atom_j c\big(p(j)\big) \sigma(j) \cdot \sigma(k) + r \cdot \sigma(k)\right) \right)\notag \\
&=\E_{c,p,r}\left(\frac{\chi(I,k)\cdot c\big(p(k)\big)}{\sqrt{1+\|\noise \|_2^2}} \: \big[P(\dico_{I})-P(\pdico_{I})- P(\atom_k) + P(\patom_k) \big]\atom_k  \right)\notag \\
& = \frac{C_r \gamma_{1,S}}{K}\: {K-1 \choose S-1}^{-1} \sum_{|I|=S, k\in I} \big[P(\patom_k) -P(\pdico_{I})\big]\atom_k.
\end{align}
We next split the sum above into a sum over the well-conditioned subsets, where $\delta_I(\dico)\leq \delta_0$, and the ill-conditioned subsets, $\delta_I(\dico)>\delta_0$,
\begin{align}\label{Evsplitsum}
\E(v)  = \frac{C_r \gamma_{1,S}}{K}\: {K-1 \choose S-1}^{-1}\left( \sum_{|I|=S, k\in I \atop
\delta(\dico_I)\leq \delta_0} \big[P(\patom_k) -P(\pdico_{I})\big]\atom_k +  \sum_{|I|=S, k\in I \atop
\delta(\dico_I)> \delta_0} \big[P(\patom_k) -P(\pdico_{I})\big]\atom_k \right).
\end{align}
We further expand the sum over the well-conditioned sets using Sublemma~\ref{sublemma4},
\begin{align}
\sum_{|I|=S, k\in I \atop \delta(\dico_I)\leq \delta_0} \big[P(\patom_k)& -P(\pdico_{I})\big]\atom_k = 
\sum_{|I|=S, k\in I \atop \delta(\dico_I)\leq \delta_0}\left( P(\dico_{I})b_k + \eta_{I,k} \right)\notag\\
&=\sum_{|I|=S, k\in I \atop \delta(\dico_I)\leq \delta_0}\left( \dico_I \dico_I^\star b_k + \left[ P(\dico_I)- \dico_I\dico_I^\star\right] b_k +\eta_{I,k} \right) \notag \\
&=\sum_{|I|=S, k\in I } \dico_I\dico_I^\star b_k - \sum_{|I|=S, k\in I \atop \delta(\dico_I)> \delta_0} \dico_I \dico_I^\star b_k + 
\sum_{|I|=S, k\in I \atop \delta(\dico_I)\leq \delta_0} \left(  \left[ P(\dico_I)-  \dico_I\dico_I^\star\right] b_k +\eta_{I,k}\right)\notag \\
&={K-2 \choose S-2} \dico \dico^\star b_k - \sum_{|I|=S, k\in I \atop \delta(\dico_I)> \delta_0} \dico_I \dico_I^\star b_k + 
\sum_{|I|=S, k\in I \atop \delta(\dico_I)\leq \delta_0} \left( \left[ P(\dico_I)- \dico_I\dico_I^\star\right] b_k +\eta_{I,k}\right),
\end{align}
where for the last equality we have used that $\ip{b_k}{\atom_k}=0$. Substituting the last expression into~\eqref{Evsplitsum}
we get,
\begin{align}\label{eq:randsign2}
\E(v)  = \frac{C_r \gamma_{1,S}}{K}\: &\left[ \frac{S-1}{K-1}\dico \dico^\star b_k + {K-1 \choose S-1}^{-1}\sum_{|I|=S, k\in I \atop \delta(\dico_I)\leq \delta_0} \left( \left[ P(\dico_I) - \dico_I\dico_I^\star\right] b_k+\eta_{I,k}\right)\right.\notag\\
&\hspace{4cm}+\left. {K-1 \choose S-1}^{-1} \sum_{|I|=S, k\in I \atop
\delta(\dico_I)> \delta_0}\left(\big[P(\patom_k) -P(\pdico_{I})\big]\atom_k - \dico_I \dico_I^\star b_k \right) \right].
\end{align}
Substituting the bound $\|P(\dico_I)-\dico_I\dico_I^\star\|_{2,2}\leq \delta(\dico_I) \leq \delta_0$ as well as the bound for $\|\eta_{I,k}\|_2$ from Sublemma~\ref{sublemma4} for the well-conditioned subsets and the bound
\begin{align}  \label{eq:randperm}
\left\| \big[P(\patom_k) -P(\pdico_{I})\big]\atom_k\right\|_2 = \| P(\pdico_I) Q(\patom_k) \atom_k \|_2 \leq  \| Q(\patom_k) \atom_k \|_2 = \sqrt{1-|\ip{\patom_k}{\atom_k}|^2}\leq \eps_k
\end{align}
for the ill-conditioned subsets finally leads to
\begin{align}\label{Ev1}
\|\E(v)\|_2 &\leq \frac{C_r \gamma_{1,S}}{K}\: \left[ \frac{S-1}{K-1}B \|b_k\|_2 + \delta_0\| b_k \|_2+\frac{2 \eps \sqrt{S}}{\sqrt{(1-\delta_0)(1-\frac{\eps^2}{2})} - 2\eps \sqrt{S} } \cdot \|b_k\|_2 + \eps  \|b_k\|_2 \right.\notag\\
&\hspace{20mm}\left.\phantom{\frac{2 \eps \sqrt{S}}{\sqrt{(1-\delta_0)(1-\frac{\eps^2}{2})} - 2\eps \sqrt{S} } } + \P(\delta(\dico_I)> \delta_0 : |I|=S, k\in I) \cdot (\eps_k + B \|b_k\|_2) \right],\notag\\
&\leq \frac{C_r \gamma_{1,S}}{K}\: \left[ \frac{SB}{K} + \delta_0+\eps + \frac{2 \eps \sqrt{S}}{\sqrt{(1-\delta_0)(1-\frac{\eps^2}{2})} - 2\eps \sqrt{S} } + (B + 1)\P(\delta(\dico_I)> \delta_0 : |I|=S, k\in I) \right]\|b_k\|_2.
\end{align}
If $\delta_S\leq \frac{1}{2}$, we choose $\delta_0=\delta_S$, which for $S\leq \frac{K}{98B}$ and $\eps \leq \frac{1}{32\sqrt{S}}$ leads to 
\begin{align}
\|\E(v)\|_2 \leq 0.611 \eps \cdot \frac{C_r \gamma_{1,S}}{K}.
\end{align}
In the non-trivial case, where $\dico$ does not have a uniform isometry constant $\delta_S\leq \frac{1}{2}$, we can estimate \eqref{Ev1} using J.~Tropp's results on the conditioning of random subdictionaries. Reformulating Theorem~12 in \cite{tr08} for our purposes we get that 
\begin{align}
\P(\delta(\dico_I)> \delta_0 : |I|=S) \leq e^{-s} \qquad\mbox{for} \qquad s= \frac{\left(e^{-1/4}\delta_0-\frac{2SB}{K}\right)^2}{144\mu^2 S},
\end{align}
whenever $e^{-1/4}\delta_0\geq \frac{2SB}{K}$, $s\geq \log(S/2 + 1)$ and $S\geq 4$. Together with the union bound,
\begin{align}
\P(\delta(\dico_I)> \delta_0 : |I|=S, k\in I)& = {K-1 \choose S-1}^{-1} \sharp\{I : \delta(\dico_I)> \delta_0, |I|=S, k\in I\}\notag\\
&\leq {K-1 \choose S-1}^{-1} \sharp\{I : \delta(\dico_I)> \delta_0, |I|=S\} = \frac{K}{S} \cdot \P(\delta(\dico_I)> \delta_0 : |I|=S),
\end{align}
this leads to 
\begin{align}
\P(\delta(\dico_I)> \delta_0 : |I|=S, k\in I) \leq \max\left\{S,\frac{K}{S}\right\} \exp\left(-\frac{\left(e^{-1/4}\delta_0-\frac{2SB}{K}\right)^2}{144\mu^2 S}\right),
\end{align}
whenever $e^{-1/4}\delta_0\geq \frac{2SB}{K}$ - in case one of the other original conditions is violated the statement is trivially true.  Using the assumption $S\leq \frac{K}{98B}$, which does not represent a hard additional constraint, considering that in order to have $\eps_{\mu,\nsigma}< 1$ we need $S\leq \frac{1}{98\mu^2}$ and that $\mu^2\geq \frac{B-1}{K-1}\approx \frac{B}{K}$, we get for $\delta_0 = \frac{1}{4}$,
\begin{align}\label{epsdelta}
\P\left(\delta(\dico_I)> \frac{1}{4} : |I|=S, k\in I \right) \leq K  \exp\left(-\frac{1}{4741\mu^2 S}\right):=\eps_\delta,
\end{align}
Substituting this bound for the choice $\delta_0 = \frac{1}{4}$ into \eqref{Ev1} and using that $\eps\leq \frac{1}{32\sqrt{S}}$ and $\eps_\delta \leq \frac{1}{48(B+1)}$ we get 
\begin{align}
\|\E(v)\|_2 &\leq 0.381 \eps \cdot \frac{C_r \gamma_{1,S}}{K} .
\end{align}
The second quantity we need to bound is the expected energy of $v= T(I,k) y \cdot \sigma(k) \cdot  \chi(I,k)$,
\begin{align}
\E(\|v\|_2^2) &=\E_{c,p,\sigma, r}\left(\frac{\chi(I,k)}{1+\|\noise\|_2^2}\cdot \Big\|  T(I,k) \Big(\sum_j \atom_j c\big(p(j)\big) \sigma(j) + r \Big)\Big\|_2^2\right)\notag \\
&=\E_{c,p,r}\left(\frac{\chi(I,k)}{1+\|\noise\|_2^2}\left(\sum_j c\big(p(j)\big)^2 \|  T(I,k) \atom_j \|^2_2 + \|T(I,k) r\|^2_2 \right) \right)\notag \\
&=\E_{p,r}\left(\frac{\chi(I,k)}{1+\|\noise\|_2^2}\left(\frac{\gamma_{2,S}}{S} \sum_{j\in I}  \|  T(I,k) \atom_j \|^2_2  +  \frac{1-\gamma_{2,S}}{K-S}\sum_{j\notin I} \|  T(I,k) \atom_j \|^2_2+ \|T(I,k) r\|^2_2 \right) \right),\notag\\
&\leq \E_{p}\left(\chi(I,k)\left(\frac{\gamma_{2,S}}{S} \sum_{j\in I}  \|  T(I,k) \atom_j \|^2_2  +  \frac{1-\gamma_{2,S}}{K-S}\sum_{j\notin I} \|  T(I,k) \atom_j \|^2_2+ \E_r\left( \|T(I,k) r\|_2^2\right) \right) \right)\label{lemma4Ev2} .
\end{align}
We first estimate the two sums above given that $k\in I$. Note that we always have $\|P(\atom_k) - P(\patom_k)\|_{2,2} \leq \eps_k$ and $\|P(\atom_k) - P(\patom_k)\|_{F} \leq \sqrt{2} \eps_k$. Thus we get for the sum over $I$,
\begin{align}
\sum_{j\in I}  \|  T(I,k) \atom_j \|^2_2& \leq  \sum_{j\in I}\left(\|[P(\dico_{I})-P(\pdico_{I})]\atom_j\|_2 + \|[P(\atom_k) - P(\patom_k)]\atom_j\|_2 \right)^2\notag\\
&= \sum_{j\in I}\left(\|Q(\pdico_{I})]\atom_j\|_2 + \|[P(\atom_k) - P(\patom_k)]\atom_j\|_2 \right)^2\notag\\
& \leq \sum_{j\in I}\left(\|Q(\patom_{j})]\atom_j\|_2 + \|P(\atom_k) - P(\patom_k)\|_{2,2} \right)^2 \leq \sum_{j\in I}\left(\eps_j + \eps_k \right)^2 \leq 4S\eps^2,
\end{align}
and for the sum over the complement $I^c$, 
\begin{align}
\sum_{j\notin I}  \|  T(I,k) \atom_j \|^2_2& = \|T(I,k)\dico_{I^c}\|^2_F 
\leq \|T(I,k)\|_F^2  \|\dico_{I^c}\|^2_{2,2} \leq B \|T(I,k)\|_F^2.
\end{align}
To estimate the noise term in \eqref{lemma4Ev2} we use the singular value decomposition of $T(I,k) = U D V^\star $,
\begin{align}
\E\left( \|T(I,k) r\|_2^2\right) = \E\left( \|DV^\star r\|_2^2\right) = \E\left( \sum_i d^2_i|\ip{v_i}{r}|^2\right)\leq \sum_i d_i^2 \nsigma^2 = \nsigma^2 \|T(I,k)\|_F^2,
\end{align}
where for the inequality we have used that for a subgaussian vector $r$ with parameter $\nsigma$, the marginal $\ip{v_i}{r}$ is subgaussian with parameter $\nsigma$.
Substituting these estimates together with the bound $\|T(I,k)\|_F\leq  \|P(\dico_{I})-P(\pdico_{I})\|_F + \sqrt{2}\eps_k$ into \eqref{lemma4Ev2} we get,
\begin{align}
\E(\|v\|_2^2) 
&\leq\E_{p}\left(\chi(I,k)\left[4 \gamma_{2,S}\eps^2 + \left( \frac{B(1-\gamma_{2,S})}{K-S}+\nsigma^2\right) \left( \|P(\dico_{I})-P(\pdico_{I})\|_F + \sqrt{2}\eps_k\right)^2\right]\right).
\end{align}
As for the estimation of $\E(v)$ we now split the expectation over $p$ into the well and the ill-conditioned subsets $I=p^{-1}(\Sset)$. By Lemma A.2 in \cite{sc14}, whenever $\delta(\dico_I) \leq \delta_0$, we have 
\begin{align}
\|P(\dico_I )-P(\pdico_I) \|^2_F \leq \frac{2\| Q(\dico_I)B_I \|^2_F}{\sqrt{1-\delta_0} \left(\sqrt{1-\delta_0} - 2\|B_I \|_{F}\right) }
\end{align}
which for $\eps \leq\frac{1}{32 \sqrt{S}}$ and $\delta_0= 1/4$ (resp. $\delta_S \leq 1/2$) simplifies to 
$\|P(\dico_I )-P(\pdico_I) \|^2_F\leq  5S\eps^2 $. Together with the general estimate $\|P(\dico_I )-P(\pdico_I) \|_F \leq \sqrt{2S}$, this leads to
\begin{align}
\E(\|v\|_2^2) &\leq \frac{S}{K}\left[4 \gamma_{2,S}\eps^2 + \left( \frac{B(1-\gamma_{2,S})}{K-S}+\nsigma^2\right) \left(\sqrt{5S}\eps + \sqrt{2}\eps_k\right)^2 \right. \notag\\
&\qquad \qquad+\left. \P\left(\delta(\dico_I)> \frac{1}{4}: |I|=S, k\in I\right)  \left( \frac{B(1-\gamma_{2,S})}{K-S}+\nsigma^2\right) \left(2S+ 2\eps_k\sqrt{S} \right) \right]\notag\\
&\leq \frac{S}{K}\left[4 \gamma_{2,S}\eps^2 + 15 \eps^2 \left( \frac{SB}{K-S}(1-\gamma_{2,S})+S\nsigma^2\right) \right. \notag\\
&\qquad \qquad+\left. \P\left(\delta(\dico_I)> \frac{1}{4}: |I|=S, k\in I\right)  \left(1-\gamma_{2,S}+ d \nsigma^2\right) \frac{2B(S+1)}{K-S} \right]\notag.
\end{align}
Substituting the probability bound from \eqref{epsdelta} and assuming again that $S\leq \frac{K}{98B}$ as well as that $S \leq \frac{1}{98\nsigma^2}$ leads to the final estimate
\begin{align}
\E(\|v\|_2^2) & \leq \frac{S}{K}\left[5\eps^2 +  \frac{\eps_\delta}{32} \left(1-\gamma_{2,S}+ d \nsigma^2\right)\right].
\end{align}
Last we bound the norm of $v$ in general as
\begin{align}
\|v\|_2&= \|[P(\dico_{I})-P(\pdico_{I})- P(\atom_k) + P(\patom_k) ] y\|_2\leq 2\|y\|_2\leq 2\sqrt{B+1}.
\end{align}
In case $\gamma_{2,S}=1, \nsigma=0$ and therefore $y=\dico_I x_I$ this reduces to 
\begin{align}
\|v\|_2
&\leq\|[\dico_{I}-P(\pdico_{I})\dico_I\|_F\|x_I\|_2 +\| P(\atom_k) - P(\patom_k) \|_{2,2} \|\dico_I x_I \|_2\notag \\
&\leq \left(\sum_{i \in I} \|\atom_i - P(\pdico_{I})\atom_i\|_2^2\right)^{\frac{1}{2}} +\eps \sqrt{B} \leq \eps \left(\sqrt{S} + \sqrt{B}\right),
\end{align} 
and in case of uniform isometry constant $\delta_S(\dico) \leq 1/4$ and $\eps\leq \frac{1}{32\sqrt{S}}$ to 
\begin{align}
\|v\|_2
& \leq\|[P(\dico_{I})-P(\pdico_{I})\|_F \|y\|_2 + \|P(\atom_k) - P(\patom_k) \|_{2,2}\|y\|_2\leq \eps\sqrt{B+1}\left(\sqrt{3S} + 1\right).
\end{align}
Putting all the pieces together we get that under the assumptions in (a),
\begin{align}
&\P\left(\frac{1}{N} \left\| \sum_n \left[R^o(\pdico, y_n, k)-R^o(\dico, y_n, k) \right]\right\|_2 \geq \frac{ C_r \gamma_{1,S}}{K}(0.381\eps + t)\right)\notag\\
&\hspace{4cm}\leq \exp\left(- \frac{ t C_r \gamma_{1,S} N}{8K} \min\left\{\frac{t C_r \gamma_{1,S}}{S\left[5\eps^2 +  \eps_\delta \left(1-\gamma_{2,S}+ d \nsigma^2\right)/32\right]}, \frac{1}{3\sqrt{B+1}}\right\} +\frac{1}{4}\right)\notag\\
&\hspace{4cm}\leq \exp\left(- \frac{ t C^2_r \gamma^2_{1,S} N}{40K\max\{S,B+1\}} \min\left\{\frac{t}{\eps^2 +  \eps_\delta \left(1-\gamma_{2,S}+ d \nsigma^2\right)/160},\frac{3}{5}\right\} +\frac{1}{4}\right),\notag
\end{align} 
under the assumptions in (b), 
\begin{align}
&\P\left(\frac{1}{N} \left\| \sum_n \left[R^o(\pdico, y_n, k)-R^o(\dico, y_n, k) \right]\right\|_2 \geq \frac{ C_r \gamma_{1,S}}{K}(0.381 \eps + t)\right)\notag\\
&\hspace{4cm}\leq \exp\left(- \frac{ t C_r \gamma_{1,S} N}{8K} \min\left\{\frac{t C_r \gamma_{1,S}}{4\eps^2 S}, \frac{1}{3\eps \sqrt{S(B+1)}}\right\} +\frac{1}{4}\right)\notag\\
&\hspace{4cm}\leq \exp\left(- \frac{ t C^2_r \gamma^2_{1,S} N}{32\eps K\max\{S,B+1\}} \min\left\{\frac{t}{\eps },1\right\} +\frac{1}{4}\right),\notag
\end{align} 
and under the assumptions in (c),
\begin{align}
&\P\left(\frac{1}{N} \left\| \sum_n \left[R^o(\pdico, y_n, k)-R^o(\dico, y_n, k) \right]\right\|_2 \geq \frac{\gamma_{1,S}}{K}(0.611\eps + t)\right)\notag\\
&\hspace{4cm} \leq \exp\left(- \frac{ t \gamma^2_{1,S} N}{40\eps K\max\{S,B+1\}} \min\left\{\frac{t}{\eps},1\right\} +\frac{1}{4}\right).\notag
\end{align} 
\end{proof}

\begin{remark}
For the lemma we have used both the sign and the permutation invariance, the sign invariance in \eqref{eq:randsign2} and the permutation invariance in \eqref{eq:randperm}. As for Lemma~\eqref{lemma2} but with a lot more effort, we can use the permutation invariance instead of using the sign invariance in \eqref{eq:randsign2}. We will not go into details but via expanding the sum $T(I,k) \sum_{j \in I, j\neq k }  x(j)\atom_j$, approximating $P(\pdico_I) \approx \pdico_I\pdico_I^\star$ and keeping track of how often an atom $\atom_j$ is in the support $I$ one can show that as long as $S^2 \lesssim K$ we still have $\| E(v)\|_2 < \eps \cdot C_r \gamma_{1,S}/K$ which is the necessary ingredient for the convergence proof. An alternative criterion, that trades off permutation invariance for sign invariance, is again the one discussed in Remark~\ref{rem:lemma2}. However it is not enough to preserve Eq.~\eqref{eq:randperm}, where we need that $\| \E_{I:k\in I}\dico^I \dico_I^\star b_k \|_2 \leq \eps$. For this inequality we do not only need to avoid that two atoms $\atom_j$ and $\atom_k$ are always used in the same ratio, but that they are always used together no matter the ratio, because any two atoms $\tilde \atom_j$ and $\tilde\atom_k$ which span the same subspace have the same approximation properties. Indeed if $x(j)$ and $x(k)$ are both randomly $\pm 1/\sqrt{S}$ then $\tilde \atom_j = \atom_j + \atom_k$ and $\tilde \atom_k = \atom_j - \atom_k$ actually provide sparser approximations.
\end{remark}

\begin{sublemma}\label{sublemma4}
Let $\dico_I$ be a subdictionary of $\dico$ with $\delta(\dico_I)\leq \deltaz$ and $\pdico_I$ the corresponding subdictionary of an $\eps$-perturbation of $\pdico$, that is $d(\dico,\pdico) = \eps$. If $k\in I$ then
\begin{align}
\big[P(\patom_k) -P(\pdico_{I})\big]\atom_k = P(\dico_{I})b_k + \eta_{I,k}\quad \mbox{ with }\quad\| \eta_{I,k}\|_2 \leq \left(\frac{2 \eps \sqrt{S}}{\sqrt{(1-\delta_0)(1-\frac{\eps^2}{2})}- 2\eps \sqrt{S} }  + \eps\right) \cdot \|b_k\|_2.
\end{align}
\end{sublemma}
\begin{proof}
If $\delta(\dico_I)\leq \deltaz$ we can use the expression for $P(\pdico_I)$ developed in Lemma A.2 of \cite{sc14}, 
\begin{align}\label{pertproj}
P(\pdico_I)&=\big( \dico_I+ Q(\dico_I)B_I M_I\big)(\dico_I^\star \dico_I)^{-1}\left(\I_S +\sum_{i=1}^\infty (-R_I)^i\right)\big( \dico_I+ Q(\dico_I) B_I M_I \big)^\star,\notag \\
\mbox{with}&  \qquad M_I= \I_S + \sum_{i=1}^\infty (- \dico_I^\dagger B_I)^i \qquad \mbox{and}  \qquad R_I=M_I^{\star} B_I^\star Q(\dico_I)B_I M_I(\dico_I^\star \dico_I)^{-1}
\end{align}
to get $P(\patom_k)\atom_k = \alpha_k^2 (\atom_k + b_k)$ and
\begin{align}
P(\pdico_{I})\atom_k &= \big( \dico_I+ Q(\dico_I)B_I M_I\big) (\dico_I^\star \dico_I)^{-1} \left(\I_S +\sum_{i=1}^\infty (-R_I)^i\right)\dico_I^\star \atom_k\notag \\
&= \atom_k +  Q(\dico_I)B_I M_I (\dico_I^\star \dico_I)^{-1}\dico_I^\star \atom_k +  \big( \dico_I+ Q(\dico_I)B_I M_I\big) (\dico_I^\star \dico_I)^{-1} \sum_{i=1}^\infty (-R_I)^i \dico_I^\star \atom_k \notag\\ 
&= \atom_k +  Q(\dico_I)B_I \left(\I_S + \sum_{i=1}^\infty (- \dico_I^\dagger B_I)^i \right){e_k}_{| I} +  \big( \dico_I+ Q(\dico_I)B_I M_I\big) (\dico_I^\star \dico_I)^{-1} \sum_{i=1}^\infty (-R_I)^i \dico_I^\star \atom_k \notag\\ 
&=\atom_k + b_k - P(\dico_I)b_k + Q(\dico_I)B_I\sum_{i=1}^\infty (- \dico_I^\dagger B_I)^i {e_k}_{| I} + \left( \dico_I+ Q(\dico_I)B_I M_I\right)  (\dico_I^\star \dico_I)^{-1}\sum_{i=1}^\infty (-R_I)^i  \dico_I^\star \atom_k.\notag
\end{align}
Subtracting the projections we see that all that remains to do is to estimate the size of
\begin{align}
\eta_{I,k}:=Q(\dico_I)B_I M_I (\dico_I^\dagger B_I) {e_k}_{| I} - \big( (\dico^\dagger_I)^\star+ Q(\dico_I)B_I M_I (\dico_I^\star \dico_I)^{-1}\big) \sum_{i=1}^\infty (-R_I)^i  \dico_I^\star \atom_k - \frac{\omega_k^2}{\alpha_k} \patom_k.
\end{align}
Using standard bounds for matrix vector products and the identity $\|(\dico_I^\star \dico_I)^{-1}\|_{2,2}= \|\dico_I^\dagger\|_{2,2}^2$ we get
\begin{align*}
\|\eta_{I,k}\|_2&\leq \| B_I M_I\|_{2,2} \|\dico_I^\dagger b_k\|_2 + \left(\|  \dico^\dagger_I\|_{2,2}+\| B_I M_I\|_{2,2} \|\dico_I^\dagger\|_{2,2}^2\right) \sum_{i=0}^\infty \| R_I\|_{2,2}^i \|R_I \dico_I^\star \atom_k\|_2 + \frac{\omega_k^2}{\alpha_k} \\
&\leq \| B_I M_I\|_{2,2} \|\dico_I^\dagger\|_{2,2}\| b_k\|_2 +
 \left(\|  \dico^\dagger_I\|_{2,2}+\| B_I M_I\|_{2,2} \|\dico_I^\dagger\|_{2,2}^2\right) \sum_{i=0}^\infty \left(\|\dico_I^\dagger\|^2_{2,2}\| B_I M_I\|^2_{2,2}\right)^{i} \|R_I \dico_I^\star \atom_k\|_2 + \frac{\omega_k^2}{\alpha_k}.
\end{align*}
We next expand $R_I \dico_I^\star \atom_k$ remembering the definition of $R_I$ and $M_I$ as
\begin{align*}
R_I \dico_I^\star \atom_k&= M_I^{\star} B_I^\star Q(\dico_I)B_I \left( \I_S + \sum_{i=1}^\infty (- \dico_I^\dagger B_I)^i \right)(\dico_I^\star \dico_I)^{-1} \dico_I^\star \atom_k\\
&=M_I^{\star} B_I^\star Q(\dico_I) \left( \I_d + \sum_{i=1}^\infty (-B_I \dico_I^\dagger )^i \right)B_I{e_k}_{| I} =M_I^{\star} B_I^\star Q(\dico_I) \left( \I_d + \sum_{i=1}^\infty (-B_I \dico_I^\dagger )^i \right)b_k\end{align*}
to get $$\|R_I \dico_I^\star \atom_k\|_2\leq \|B_I M_I\|_{2,2}\left( 1-\|B_I\|_{2,2}\|\dico_I^\dagger\|_{2,2}\right)^{-1}\|b_k\|_2.$$
Substituting this estimate together with the bound $\| M_I\|_{2,2} \leq \left( 1-\|B_I\|_{2,2}\|\dico_I^\dagger\|_{2,2}\right)^{-1}$ into the above bound for $\|\eta_{I,k}\|_2$, resolving the sums and fractions and noting that $\|b_k\|_2=\frac{\omega_k}{\alpha_k}$ leads to,
\begin{align*}
\|\eta_{I,k}\|_2&\leq  \left(\frac{2 \|B_I\|_{2,2}}{\|\dico_I^\dagger\|_{2,2}^{-1}-2\|B_I\|_{2,2}} + \omega_k \right) \cdot \|b_k\|_2 .
\end{align*}
To get to the final statement we use the bounds $ \|B_I\|^2_{2,2}\leq\|B_I\|^2_F\leq S \eps^2/(1-\eps^2/2)$ and $\|\dico_I^\dagger\|_{2,2}^{-1}\geq\sqrt{1-\delta(\dico_I)}\geq \sqrt{1-\delta_0}$.
\end{proof}

\begin{lemma}\label{lemma_rescale}
If for two vectors $\patom, \atom$, where $\|\atom\|_2=1$, and two scalars $0<t<s$ we have,
$\|\patom - s \atom\|^2_2\leq t^2$ then 
\begin{align}
\left\| \frac{\patom}{\|\patom\|_2} - \atom \right\|^2_2 \leq 2-2\sqrt{1-\frac{t^2}{s^2}}.
\end{align}
\end{lemma}
\begin{proof}
Writing $\patom = \alpha \atom + \omega z$ for some unit norm vector $z$ with $\ip{z}{\atom}=0$ we can reformulate
the initial constraint $\|\patom - s \atom\|^2_2\leq t^2$ to $(\alpha - s)^2 + \omega^2\leq t^2$, while the quantity whose maximal size we have to estimate becomes
\begin{align}
\left\| \frac{\patom}{\|\patom\|_2} - \atom \right\|_2^2=2-2\frac{\alpha}{\sqrt{\alpha^2 + \omega^2}}.
\end{align}
Solving the resulting maximisation problem we get that the maximum is attained at $\alpha= \frac{s^2-t^2}{s}$ and $\omega=\frac{t}{s}\sqrt{s^2-t^2}$ and that therefore
\begin{align}
\left\| \frac{\patom}{\|\patom\|_2} - \atom \right\|_2^2\leq 2-2\sqrt{1-\frac{t^2}{s^2}}.
\end{align}
\end{proof}

\bibliography{/Users/karin/Desktop/latexnotes/karinbibtex}

\begin{thebibliography}{10}

\bibitem{aganjaneta13}
A.~Agarwal, A.~Anandkumar, P.~Jain, P.~Netrapalli, and R.~Tandon.
\newblock Learning sparsely used overcomplete dictionaries via alternating
  minimization.
\newblock In {\em COLT 2014 (arXiv:1310.7991)}, 2014.

\bibitem{aganne13}
A.~Agarwal, A.~Anandkumar, and P.~Netrapalli.
\newblock Exact recovery of sparsely used overcomplete dictionaries.
\newblock In {\em COLT 2014 (arXiv:1309.1952)}, 2014.

\bibitem{ahelbr06}
M.~Aharon, M.~Elad, and A.M. Bruckstein.
\newblock {K}-{S}{V}{D}: An algorithm for designing overcomplete dictionaries
  for sparse representation.
\newblock {\em IEEE Transactions on Signal Processing.}, 54(11):4311--4322,
  November 2006.

\bibitem{arbhgema14}
S.~Arora, A.~Bhaskara, R.~Ge, and T.~Ma.
\newblock More algorithms for provable dictionary learning.
\newblock {\em arXiv:1401.0579}, 2014.

\bibitem{argemamo15}
S.~Arora, R.~Ge, T.~Ma, and A.~Moitra.
\newblock Simple, efficient, and neural algorithms for sparse coding.
\newblock In {\em COLT 2015 (arXiv:1503.00778)}, 2015.

\bibitem{argemo13}
S.~Arora, R.~Ge, and A.~Moitra.
\newblock New algorithms for learning incoherent and overcomplete dictionaries.
\newblock In {\em COLT 2014 (arXiv:1308.6273)}, 2014.

\bibitem{bakest14}
B.~Barak, J.A. Kelner, and D.~Steurer.
\newblock Dictionary learning and tensor decomposition via the sum-of-squares
  method.
\newblock In {\em STOC 2015 (arXiv:1407.1543)}, 2015.

\bibitem{bennett62}
G.~Bennett.
\newblock Probability inequalities for the sum of independent random variables.
\newblock {\em Journal of the American Statistical Association},
  57(297):33--45, March 1962.

\bibitem{blda08}
T.~Blumensath and M.E. Davies.
\newblock Iterative thresholding for sparse approximations.
\newblock {\em Journal of Fourier Analysis and Applications}, 14(5-6):629--654,
  2008.

\bibitem{blda09}
T.~Blumensath and M.E. Davies.
\newblock {I}terative {H}ard {T}hresholding for compressed sensing.
\newblock {\em Applied Computational Harmonic Analysis}, 27(3):265--274, 2009.

\bibitem{carota06}
{E}. {C}and{\`e}s, {J}. {R}omberg, and {T}. {T}ao.
\newblock {R}obust uncertainty principles: exact signal reconstruction from
  highly incomplete frequency information.
\newblock {\em {IEEE} {T}ransactions on {I}nformation {T}heory},
  52(2):489--509, 2006.

\bibitem{ch03}
O.~Christensen.
\newblock {\em An Introduction to Frames and Riesz Bases}.
\newblock {B}irkh{\"a}user, 2003.

\bibitem{doelte06}
D.L. Donoho, M.~Elad, and V.N. Temlyakov.
\newblock Stable recovery of sparse overcomplete representations in the
  presence of noise.
\newblock {\em {IEEE} {T}ransactions on {I}nformation {T}heory}, 52(1):6--18,
  January 2006.

\bibitem{enaahu99}
K.~Engan, S.O. Aase, and J.H. Husoy.
\newblock Method of optimal directions for frame design.
\newblock In {\em {ICASSP}99}, volume~5, pages 2443--2446, 1999.

\bibitem{olsfield96}
D.J. Field and B.A. Olshausen.
\newblock Emergence of simple-cell receptive field properties by learning a
  sparse code for natural images.
\newblock {\em Nature}, 381:607--609, 1996.

\bibitem{fo11}
{S.} {F}oucart.
\newblock Hard thresholding pursuit: An algorithm for compressive sensing.
\newblock {\em SIAM Journal on Numerical Analysis}, 49(6):2543--2563, 2011.

\bibitem{gewawrXX}
Q.~Geng, H.~Wang, and J.~Wright.
\newblock On the local correctness of $\ell^1$-minimization for dictionary
  learning.
\newblock {\em arXiv:1101.5672}, 2011.

\bibitem{gethci05}
P.~Georgiev, F.J. Theis, and A.~Cichocki.
\newblock Sparse component analysis and blind source separation of
  underdetermined mixtures.
\newblock {\em {IEEE} {T}ransactions on Neural Networks}, 16(4):992--996, 2005.

\bibitem{bagrje14}
R.~Gribonval, R.~Jenatton, and F.~Bach.
\newblock Sparse and spurious: dictionary learning with noise and outliers.
\newblock {\em {I}{E}{E}{E} {T}ransactions on {I}nformation {T}heory},
  61(11):6298--6319, 2015.

\bibitem{grjebaklse13}
R.~Gribonval, R.~Jenatton, F.~Bach, M.~Kleinsteuber, and M.~Seibert.
\newblock Sample complexity of dictionary learning and other matrix
  factorizations.
\newblock {\em {I}{E}{E}{E} {T}ransactions on {I}nformation {T}heory},
  61(6):3469--3486, 2015.

\bibitem{grsc10}
R.~Gribonval and K.~Schnass.
\newblock Dictionary identifiability - sparse matrix-factorisation via
  $l_1$-minimisation.
\newblock {\em {IEEE} {T}ransactions on {I}nformation {T}heory},
  56(7):3523--3539, July 2010.

\bibitem{gr11}
D.~Gross.
\newblock Recovering low-rank matrices from few coefficients in any basis
  recovering low-rank matrices from few coefficients in any basis recovering
  low-rank matrices from few coefficients in any basis.
\newblock {\em {IEEE} {T}ransactions on {I}nformation {T}heory},
  57(3):1548--1566, 2011.

\bibitem{ho16}
E.~H\"ock.
\newblock Hard thresholding pursuit for sparse approximation.
\newblock {BSc} thesis, University of Innsbruck, 2016.

\bibitem{jateka14}
P.~Jain, A.~Tewari, and P.~Kar.
\newblock On iterative hard thresholding methods for high-dimensional
  m-estimation.
\newblock In {\em NIPS14 (arXiv:14105137)}, 2014.

\bibitem{juelgo14}
A.~Jung, Y.~Eldar, and N.~G\"{o}rtz.
\newblock Performance limits of dictionary learning for sparse coding.
\newblock In {\em EUSIPCO14 (arXiv:1402.4078)}, pages 765 -- 769, 2014.

\bibitem{kreutz03}
K.~Kreutz-Delgado, J.F. Murray, B.D. Rao, K.~Engan, T.~Lee, and T.J. Sejnowski.
\newblock Dictionary learning algorithms for sparse representation.
\newblock {\em Neural Computations}, 15(2):349--396, 2003.

\bibitem{krra00}
K.~Kreutz-Delgado and B.D. Rao.
\newblock {FOCUSS}-based dictionary learning algorithms.
\newblock In {\em SPIE 4119}, 2000.

\bibitem{kugr14}
R.~Kueng and D.~Gross.
\newblock Ripless compressed sensing from anisotropic measurements.
\newblock {\em Linear Algebra and its Applications}, 441:110--123, 2014.

\bibitem{leta91}
M.~Ledoux and M.~Talagrand.
\newblock {\em {P}robability in {B}anach spaces. {I}soperimetry and processes.}
\newblock {S}pringer-{V}erlag, {B}erlin, {H}eidelberg, {N}ew{Y}ork, 1991.

\bibitem{lese00}
M.~S. Lewicki and T.~J. Sejnowski.
\newblock Learning overcomplete representations.
\newblock {\em Neural Computations}, 12(2):337--365, 2000.

\bibitem{mabapo12}
J.~Mairal, F.~Bach, and J.~Ponce.
\newblock Task-driven dictionary learning.
\newblock {\em IEEE Transactions on Pattern Analysis and Machine Intelligence},
  34(4):791--804, 2012.

\bibitem{mabaposa10}
J.~Mairal, F.~Bach, J.~Ponce, and G.~Sapiro.
\newblock Online learning for matrix factorization and sparse coding.
\newblock {\em Journal of Machine Learning Research}, 11:19--60, 2010.

\bibitem{masazi08}
J.~Mairal, F.~Bach, J.~Ponce, G.~Sapiro, and A.~Zisserman.
\newblock Discriminative learned dictionaries for local image analysis.
\newblock IMA Preprint Series 2212, University of Minnesota, 2008.

\bibitem{mapo10}
A.~Maurer and M.~Pontil.
\newblock {K}-dimensional coding schemes in {H}ilbert spaces.
\newblock {\em {IEEE} {T}ransactions on {I}nformation {T}heory},
  56(11):5839--5846, 2010.

\bibitem{megr12}
N.A. Mehta and A.G. Gray.
\newblock On the sample complexity of predictive sparse coding.
\newblock {\em arXiv:1202.4050}, 2012.

\bibitem{nasc16}
V.~Naumova and K.~Schnass.
\newblock Dictionary learning from incomplete data, {P}art~{I} algorithms.
\newblock {\em in preparation}, 2016.

\bibitem{pl07}
M.D. Plumbley.
\newblock Dictionary learning for $\ell_1$-exact sparse coding.
\newblock In M.E. Davies, C.J. James, and S.A. Abdallah, editors, {\em
  International Conference on Independent Component Analysis and Signal
  Separation}, volume 4666, pages 406--413. Springer, 2007.

\bibitem{rubrel10}
R.~Rubinstein, A.~Bruckstein, and M.~Elad.
\newblock Dictionaries for sparse representation modeling.
\newblock {\em Proceedings of the IEEE}, 98(6):1045--1057, 2010.

\bibitem{sc14}
K.~Schnass.
\newblock On the identifiability of overcomplete dictionaries via the
  minimisation principle underlying {K-SVD}.
\newblock {\em Applied Computational Harmonic Analysis}, 37(3):464--491, 2014.

\bibitem{sc14b}
K.~Schnass.
\newblock Local identification of overcomplete dictionaries.
\newblock {\em Journal of Machine Learning Research (arXiv:1401.6354)},
  16(Jun):1211--1242, 2015.

\bibitem{sc16}
K.~Schnass.
\newblock Sequential dictionary learning with model selection.
\newblock {\em in preparation}, 2016.

\bibitem{scva07}
K.~Schnass and P.~Vandergheynst.
\newblock Average performance analysis for thresholding.
\newblock {\em IEEE Signal Processing Letters}, 14(11):828--831, 2007.

\bibitem{sken10}
K.~Skretting and K.~Engan.
\newblock Recursive least squares dictionary learning algorithm.
\newblock {\em {IEEE} {T}ransactions on {S}ignal {P}rocessing},
  58(4):2121--2130, April 2010.

\bibitem{spwawr12}
D.~Spielman, H.~Wang, and J.~Wright.
\newblock Exact recovery of sparsely-used dictionaries.
\newblock In {\em COLT 2012 (arXiv:1206.5882)}, 2012.

\bibitem{suquwr15}
J.~Sun, Q.~Qu, and J.~Wright.
\newblock Complete dictionary recovery over the sphere.
\newblock In {\em ICML 2015 (arXiv:1504.06785)}, 2015.

\bibitem{tr08}
{J}.{A}. {T}ropp.
\newblock On the conditioning of random subdictionaries.
\newblock {\em Applied Computational Harmonic Analysis}, 25(1-24), 2008.

\bibitem{vamabr11}
D.~Vainsencher, S.~Mannor, and A.M. Bruckstein.
\newblock The sample complexity of dictionary learning.
\newblock {\em Journal of Machine Learning Research}, 12(3259-3281), 2011.

\bibitem{yablda09}
M.~Yaghoobi, T.~Blumensath, and M.E. Davies.
\newblock Dictionary learning for sparse approximations with the majorization
  method.
\newblock {\em {IEEE} {T}ransactions on {S}ignal {P}rocessing},
  57(6):2178--2191, June 2009.

\bibitem{zipe01}
M.~Zibulevsky and B.A. Pearlmutter.
\newblock Blind source separation by sparse decomposition in a signal
  dictionary.
\newblock {\em Neural Computations}, 13(4):863--882, 2001.

\end{thebibliography}
\bibliographystyle{plain}
\end{document}